%% file: iclr2027_conference.tex
\documentclass{article}
\usepackage{iclr2027_conference,times}

\input{math_commands.tex}

\usepackage{hyperref}
\usepackage{url}
\hypersetup{colorlinks=true, linkcolor=blue!55!black, citecolor=blue!55!black,
            urlcolor=blue!55!black, filecolor=blue!55!black}

\usepackage{amssymb}
\usepackage{amsthm}
\usepackage{mathtools}
\usepackage{booktabs}
\usepackage{multirow}
\usepackage{graphicx}
\usepackage[table,dvipsnames]{xcolor}
\usepackage{tikz}
\usetikzlibrary{arrows.meta,positioning,fit,backgrounds,calc,decorations.pathreplacing}
\usepackage{array}
\usepackage{caption}
\usepackage{subcaption}
\usepackage{enumitem}
\usepackage{algorithm}
\usepackage{algpseudocode}
\usepackage{float}

\graphicspath{{figures/}}

\theoremstyle{plain}
\newtheorem{theorem}{Theorem}
\newtheorem{proposition}[theorem]{Proposition}
\newtheorem{lemma}[theorem]{Lemma}

\theoremstyle{definition}
\newtheorem{definition}[theorem]{Definition}
\newtheorem{assumption}{Assumption}
\theoremstyle{remark}
\newtheorem{remark}[theorem]{Remark}

\newcommand{\Rb}{\mathbb{R}}
\newcommand{\best}[1]{\textbf{#1}}

\DeclareMathOperator{\diag}{diag}
\DeclareMathOperator{\blkdiag}{blkdiag}
\DeclareMathOperator{\spls}{softplus}
\DeclareMathOperator{\tr}{tr}
\newcommand{\Top}{\mathcal{T}}
\newcommand{\Mop}{\mathcal{M}}
\newcommand{\Fcal}{\mathcal{F}}
\newcommand{\Acal}{\mathcal{A}}
\newcommand{\up}{$\uparrow$}
\newcommand{\dn}{$\downarrow$}

\title{CAMOS: Coupled Oscillatory State-Space Model for Multimodal Clinical
Time-Series}

\author{Maxx Richard Rahman \\
Saarland University and \\
German Research Center for Artificial Intelligence (DFKI) \\
Saarbr\"{u}cken, Germany \\
\And
Mostafa Hammouda \\
Saarland University \\
Saarbr\"{u}cken, Germany \\
\AND
Wolfgang Maass \\
Saarland University and \\
German Research Center for Artificial Intelligence (DFKI) \\
Saarbr\"{u}cken, Germany
}

\iclrfinalcopy 
\begin{document}

\maketitle
\lhead{Preprint}

\begin{abstract}
Longitudinal clinical cohorts are multimodal, irregularly sampled and
pervasively incomplete: in ADNI, positron emission tomography and cerebrospinal
fluid assays are absent from roughly half of all visits. Linear state-space
models handle irregular sampling gracefully but treat a missing modality by
masking the \emph{input}, leaving the transition operator untouched. We prove
that this is a representational limitation: the latent state of any linear
state-space layer whose transition operator does not depend on the availability
pattern is an \emph{additive} function of the availability indicators, so no
such layer can represent an interaction between two modalities being jointly
present or jointly absent. We propose CAMOS, which gives each modality a bank of
second-order oscillators coupled through a matrix that sits \emph{inside} the
differential equation and is gated by availability, so the transition operator
itself becomes a function of which measurements were taken. Coupling invalidates
the analysis of uncoupled oscillatory models, and we restore it: a per-channel
Gershgorin budget makes the effective stiffness positive definite uniformly over
all $2^M$ availability patterns and all gaps, an energy argument charges
amplification to availability transitions rather than sequence length, and a
channel factorization preserves exact associative parallel scans. On ADNI, CAMOS
outperforms uncoupled oscillatory state-space models and clinical fusion models
on same-visit staging, landmark prediction and longitudinal forecasting, and under zero-shot transfer to OASIS-3
it is the only model that avoids collapse to the majority class.
\end{abstract}

\section{Introduction}

State-space models (SSMs) have been widely adopted for learning on long sequences \citep{gu2022s4,smith2023s5,orvieto2023lru,gu2024mamba}, and they suit
clinical time series because they discretize a continuous-time system and so
treat a variable step size as a parameter rather than an approximation.
Oscillatory SSMs sharpen this further, deriving stability from a nonnegative
diagonal state matrix alone and admitting associative parallel scans in
logarithmic depth \citep{rusch2025linoss,boyer2025dlinoss,blelloch1990prefix}.
However, clinical cohorts pose a second difficulty alongside irregular
sampling, i.e., measurements arrive as several modalities ordered by clinical need,
cost and invasiveness, so each has its own gaps, while the underlying processes
are coupled. In Alzheimer's disease, amyloid deposition precedes tau pathology,
then neurodegeneration, then cognitive decline
\citep{hardy2002amyloid,jack2010cascade,jack2013updated}, yet in ADNI
\citep{petersen2010adni} imaging and cognitive testing occur at most visits
whereas amyloid PET and cerebrospinal fluid assays are far rarer, so the cascade
should be tracked through a missingness process that is neither uniform nor
ignorable \citep{rubin1976missing}.

Existing models treat that missingness as an input phenomenon. An absent
measurement is handled by zeroing or masking the forcing term, sometimes with an
appended indicator \citep{che2018grud}, while the transition operator, which
propagates latent state across time, is left untouched. Unrolling such a layer
shows the cost: the latent state is a sum of per-modality, per-step
contributions, each scaled by its own availability indicator, so the effect of
one modality being present cannot depend on whether another was present at a
different step. Joint availability is
nevertheless informative, since whether an unverified amyloid estimate should
influence a tau prediction depends on whether amyloid was measured, not merely
on its value. Multimodal fusion methods provide a solution to this additive class
\citep{hayat2022medfuse,zhang2022m3care,wang2023shaspec,yao2024drfuse}, but they
fuse after the temporal model has been applied, so nothing lets one modality shape
another's trajectory across a long gap.

In this paper, we place availability inside the dynamics. Each modality
receives a bank of second-order oscillators, and modalities drive one another
through a coupling matrix in the differential equation itself, gated by
availability. Because the gate multiplies a matrix in the homogeneous part, the
transition operator becomes a function of which measurements were taken, and the
additivity obstruction disappears. However, coupling invalidates the analysis
of uncoupled oscillatory models in three places: i) stability is a property of the
effective stiffness, which no condition on the individual coupling blocks alone
controls; ii) parallel evaluation relies on a diagonal state matrix, which coupling
destroys; and iii) a per-step spectral radius bound no longer implies boundedness
once the operator varies over time. Our contributions, each addressing one of them are:

\begin{itemize}[itemsep=2pt,topsep=3pt,leftmargin=*]
\item We introduce CAMOS, a coupled oscillatory state-space model whose
  transition operator is gated by the availability pattern and factored
  channel-wise, preserving exact associative parallel scans.
\item We prove that every availability-independent linear state-space layer is
  additive in the availability indicators while CAMOS is not.
\item We perform empirical evaluation on ADNI and OASIS-3 where CAMOS outperforms uncoupled         oscillatory SSMs and
  clinical fusion models on same-visit staging, landmark prediction and longitudinal forecasting.
\end{itemize}

\section{Related Work}

\paragraph{Structured, oscillatory and coupled state-space models.} S4, S5, LRU
and Mamba established structured and selective state matrices solved by FFT or
associative scans
\citep{gu2020hippo,gu2022s4,smith2023s5,orvieto2023lru,gu2024mamba,dao2024mamba2}.
LinOSS \citep{rusch2025linoss} instead discretizes a second-order oscillator
system, obtaining stability from a nonnegative diagonal state matrix alone and
universality within continuous causal operators; its two variants, LinOSS-IM and
LinOSS-IMEX, differ in whether the oscillator system is integrated by an
implicit or a symplectic implicit-explicit scheme, and D-LinOSS
\citep{boyer2025dlinoss} decouples damping from frequency. These are the direct
parents of CAMOS and our primary baselines. All of them take a single monolithic
input, so concatenating modalities leaves the state matrix shared and the
modality partition invisible to the dynamics, and no object in the model
represents one modality driving another; they correspond to a zero coupling with
a dense input matrix. 
Coupled Mamba \citep{li2024coupledmamba} is the closest on the oscillator side,
coupling the hidden-state chains of several modalities, but it couples a
discrete first-order recurrence rather than a continuous-time second-order
system, so an irregular gap has no natural meaning.

\paragraph{Irregular sampling and missing modalities.} Irregular sampling is
handled by GRU-D, ODE-based models, controlled differential equations, mTAN and
Raindrop
\citep{che2018grud,chen2018neuralode,rubanova2019latentode,debrouwer2019gruodebayes,kidger2020ncde,walker2024logncde,shukla2021mtan,zhang2022raindrop},
all of which are sequential in the sequence length and none of which gates a
coupled dynamical system by availability. Clinical fusion methods instead handle
missing modalities at the level of \emph{representations}
\citep{ma2021smil,zhang2022m3care,wang2023shaspec,yao2024drfuse,hayat2022medfuse,joze2020mmtm,polsterl2021daft},
fusing after the temporal model has be applied. MedFuse \citep{hayat2022medfuse} learns
a representation of the missing modality inside an LSTM fusion, and DrFuse
\citep{yao2024drfuse} disentangles shared from modality-specific
representations so that the shared component survives an absent modality. Both
are nonlinear in the availability pattern, which places them outside the
additive class, and both are our second family of baselines.

 \begin{figure}[ht]
    \centering
    \includegraphics[width=\textwidth]{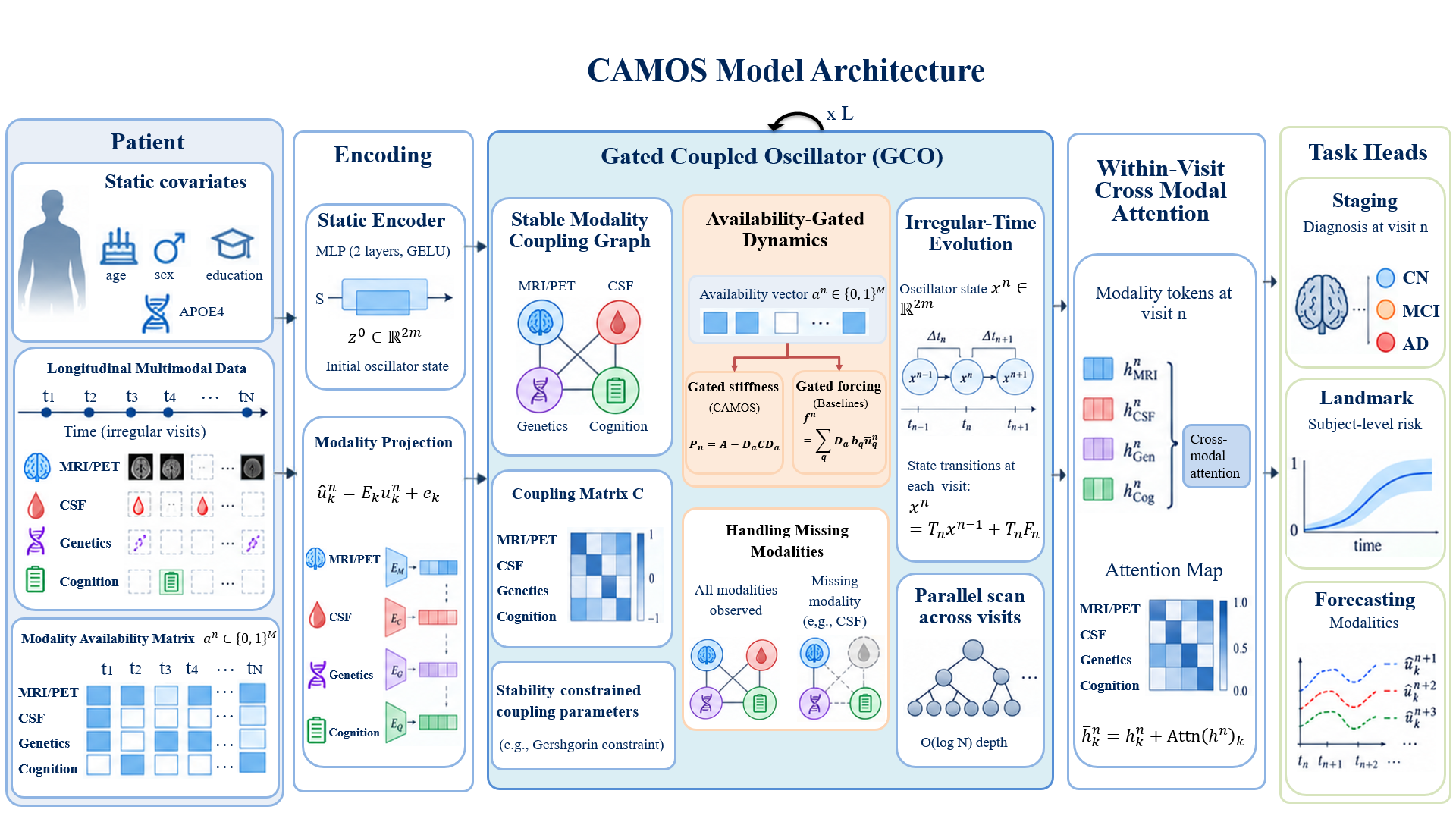}
    \caption{CAMOS architecture consists of static and longitudinal multimodal patient data through availability-gated coupled oscillators, irregular-time dynamics, and cross-modal attention.}
    \label{fig:camos}
\end{figure}

\section{Problem Formulation}
\label{sec:problem}

Let us consider a subject contribute $N$ visits at
strictly increasing times $t_1 < \dots < t_N$ with gaps
$\Delta t_n := t_n - t_{n-1} > 0$ in years, where $\Delta t_1 := t_1 - t_0$ and
$t_0$ is the cohort baseline. At visit $n$, each of the $M$ modalities
$k \in [M] := \{1,\dots,M\}$ contributes a feature vector
$u_k^n \in \Rb^{p_k}$ and a binary availability indicator $a_k^n \in \{0,1\}$,
equal to one exactly when modality $k$ was measured, with $u_k^n$ undefined and
never read when $a_k^n = 0$. We write
$a^n = (a_1^n,\dots,a_M^n) \in \{0,1\}^M$ for the \emph{availability pattern},
$a^{1:N}$ for its sequence, $m := Md$ for the stacked latent width, and
$D_{a} := \blkdiag(a_1 I_d, \dots, a_M I_d) \in \Rb^{m \times m}$ for the gate
matrix. Each subject also has static covariates $s \in \Rb^{p_s}$. The three prediction targets are: i) \textit{same-visit staging}: a visit-level ordinal
label $\ell^n \in \{\mathrm{CN},\mathrm{MCI},\mathrm{AD}\}$ predicted as
$\hat p^n$ in the $2$-simplex; ii) \textit{landmark prediction}: a subject-level binary
label $q \in \{0,1\}$ for progression from mild cognitive impairment (MCI) to
Alzheimer's disease (AD) within three years, predicted as $\hat q \in [0,1]$;
and iii) \textit{longitudinal forecasting}: the values $u_k^{n+h}$ at horizon
$h \ge 1$, predicted as $\hat u_k^{n+h}$ and scored only where
$a_k^{n+h} = 1$. The model is a map:
$  \Phi:\ \bigl(\{u_k^n\},\, a^{1:N},\, \{t_n\},\, s\bigr)
  \;\longmapsto\;
  \bigl(\{\hat p^n\},\, \hat q,\, \{\hat u_k^{n+h}\}\bigr),
$
which uses $a^{1:N}$ as a genuine input rather than merely as a mask on $u$.

\section{Proposed CAMOS Model}
\label{sec:method}

CAMOS is a stack of $L$ gated coupled oscillator (GCO) blocks between an encoder
and three task heads,
$ 
  \Phi = \bigl\{\mathrm{Head}_\tau\bigr\}_{\tau \in \{1,2,3\}}
  \circ\ \mathrm{Attn}\ \circ\
  \mathrm{GCO}^{(L)} \circ \cdots \circ \mathrm{GCO}^{(1)}
  \circ\ \mathrm{Enc}.
$ 
Everything except the GCO recurrence is applied pointwise in $n$ and is
therefore trivially parallel. The GCO recurrence is evaluated by an associative
scan in $O(\log N)$ depth. All proofs are in Appendix~\ref{app:proofs}.

\subsection{Input encoders}
\label{sec:encoders}

Static covariates $s \in \Rb^{p_s}$ set the initial condition
$x^0 = (z^0, y^0)^\top \in \Rb^{2m}$ of the oscillator bank through a two-layer
perceptron with GELU activation \citep{hendrycks2016gelu}, rather than adding an
offset at the readout, which encodes the hypothesis that baseline demographics
modulate the phase and amplitude of the cascade rather than shifting outcomes
additively. Modalities are measured on incommensurable scales, so each is
projected affinely into a shared width $h$,
\begin{equation}
  \tilde u_k^n = E_k\, u_k^n + e_k \in \Rb^{h},
  \qquad E_k \in \Rb^{h \times p_k},\ e_k \in \Rb^{h},
  \label{eq:proj}
\end{equation}
in parallel across visits and evaluated only when $a_k^n = 1$; the value of
$\tilde u_k^n$ is irrelevant when $a_k^n = 0$ because the forcing is gated in
\eqref{eq:gated}.

\subsection{Coupled multimodal oscillators, and why the coupling must be
budgeted}
\label{sec:cmo}

Each modality $k$ carries a bank of $d$ oscillators obeying
\begin{equation}
  z_k'(t) = -A_k\, y_k(t) + \sum_{j \neq k} C_{kj}\, y_j(t)
            + B_k\, \tilde u_k(t),
  \qquad
  y_k'(t) = z_k(t),
  \label{eq:ode}
\end{equation}
with $A_k, C_{kj} \in \Rb^{d\times d}$ and $B_k \in \Rb^{d \times h}$. Stacking
$y = (y_1^\top,\dots,y_M^\top)^\top \in \Rb^{m}$ and likewise for $z$,
\eqref{eq:ode} becomes $y'' = -P y + B\tilde u$ with
\begin{equation}
\begin{aligned}
  P &:= A - C,
  \qquad
  A := \blkdiag(A_1,\dots,A_M), \\
  [C]_{kj} &:= C_{kj}\quad (k \neq j),
  \qquad
  [C]_{kk} := 0,
  \qquad
  B := \blkdiag(B_1,\dots,B_M).
\end{aligned}
\label{eq:P}
\end{equation}
Stability, discretization and cost are properties of the \emph{effective
stiffness} $P$ rather than of the individual blocks, which invalidates the
natural first guess.

\begin{remark}[Block-level conditions do not give stability]
\label{rem:psdfails}
Parameterizing each block as $C_{kj} = Q_{kj}Q_{kj}^\top \succeq 0$ does
\emph{not} make $P \succeq 0$. Under Assumption~\ref{as:sym} the assembled $C$
is symmetric with zero diagonal blocks, so $\tr(C) = 0$ and $C$ has a strictly
negative eigenvalue whenever $C \neq 0$. The only free bound is
$\lambda_{\min}(P) \ge \lambda_{\min}(A) - \|C\|_2$, which can be negative
whenever $\|C\|_2 > \lambda_{\min}(A)$; then $I + \Delta t^2 P$ can be singular,
the implicit solve fails, and the transition operator can have spectral radius
exceeding one. A \emph{joint} condition on $A$ and $C$ is required.
\end{remark}

\paragraph{Channel factorization.} We take $A_k$ and $C_{kj}$ diagonal,
$A_k = \diag(\alpha_{k,1},\dots,\alpha_{k,d})$ and
$C_{kj} = \diag(c_{kj,1},\dots,c_{kj,d})$, so coupling mixes modalities within
each latent channel but not across channels. With $\Pi$ the permutation
$\Pi\, e_{(k-1)d + r} = e_{(r-1)M + k}$, reindexing from (modality-major,
channel-minor) to (channel-major, modality-minor),
\begin{equation}
  \Pi\, P\, \Pi^\top = \blkdiag\bigl(P^{(1)},\dots,P^{(d)}\bigr),
  \qquad
  \bigl[P^{(r)}\bigr]_{kk} = \alpha_{k,r}, \quad
  \bigl[P^{(r)}\bigr]_{kj} = -\,c_{kj,r}\ \ (k \neq j),
  \label{eq:channelblock}
\end{equation}
so the coupled system decomposes into $d$ independent networks of $M$
oscillators, each carrying its own weighted modality graph
$P^{(r)} \in \Rb^{M\times M}$. Appendix~\ref{app:dense} shows that a dense
$C_{kj}$ raises the per-step cost from $O(dM^3)$ to $O(d^3M^3)$, and that a
low-rank $C_{kj}$ does not help, because low-rank-plus-diagonal structure is not
closed under the scan combine.

\paragraph{Symmetry and Gershgorin budget.} We impose two assumptions.
\begin{assumption}[Symmetry]\label{as:sym}
$c_{kj,r} = c_{jk,r}$ for all $k,j \in [M]$, $r \in [d]$, so every $P^{(r)}$ is
symmetric and $P = P^\top$.
\end{assumption}
\begin{assumption}[Gershgorin budget]\label{as:budget}
There is $\epsilon \in (0,1]$ such that
$\sum_{j\neq k} |c_{kj,r}| \le (1-\epsilon)\,\alpha_{k,r}$ for all
$k \in [M]$, $r \in [d]$, with $\alpha_{k,r} > 0$.
\end{assumption}

Both hold by construction under the following parameterization, which needs no
projection step, no power iteration, and is differentiable everywhere. Let
$\hat\alpha_{k,r}$ and $\hat c_{kj,r} = \hat c_{jk,r}$ be free parameters
(parameterize the strict upper triangle and mirror it) and set, with $\delta > 0$
a small constant,
\begin{equation}
  \alpha_{k,r} = \spls(\hat\alpha_{k,r}),
  \quad
  \sigma_{k,r} = \sum_{l \neq k} |\hat c_{kl,r}|,
  \quad
  c_{kj,r}
  = \frac{(1-\epsilon)\,\min\{\alpha_{k,r},\alpha_{j,r}\}}
         {\max\{\sigma_{k,r},\,\sigma_{j,r},\,\delta\}}
    \;\hat c_{kj,r}.
  \label{eq:param}
\end{equation}
Symmetry is immediate because the prefactor is symmetric in $(k,j)$, and the
budget follows termwise. Writing $\alpha_{\min} := \min_{k,r}\alpha_{k,r}$,
$\alpha_{\max} := \max_{k,r}\alpha_{k,r}$ and
$\gamma := \max_{r}\|C^{(r)}\|_2$ for the coupling strength, we set
$\mu := \epsilon\,\alpha_{\min}$ and $L_P := (2-\epsilon)\,\alpha_{\max}$, and
Assumption~\ref{as:budget} gives $\gamma \le (1-\epsilon)\alpha_{\max}$.

\begin{lemma}[Uniform positive definiteness]\label{lem:psd}
For every $a \in [0,1]^M$ the matrix $P(a) = A - D_a C D_a$ is symmetric and
$\mu I_m \preceq P(a) \preceq L_P I_m$. In particular $P(a) \succ 0$
simultaneously for all $2^M$ binary availability patterns,
$I_m + \Delta t^2 P(a)$ is symmetric positive definite for every
$\Delta t \in \Rb$, and its Cholesky factorization exists.
\end{lemma}

The proof is a per-channel Gershgorin argument using only $|a_k a_j| \le 1$, so
gating can shrink but never enlarge an off-diagonal row sum: the budget is
imposed once, on the \emph{ungated} coupling, and gating preserves it
automatically, so no availability pattern can destabilize the model.

\subsection{Availability gating and what input masking cannot represent}
\label{sec:gating}

Availability enters in exactly two places, the stiffness and the forcing:
\begin{equation}
  P_n := A - D_{a^n}\, C\, D_{a^n},
  \qquad
  f^n := \sum_{k=1}^{M} a_k^n\, \iota_k\, B_k\, \tilde u_k^n \in \Rb^m,
  \label{eq:gated}
\end{equation}
where $\iota_k : \Rb^d \to \Rb^m$ embeds into the $k$-th modality block. Since
$[D_{a^n} C D_{a^n}]_{kj} = a_k^n a_j^n C_{kj}$, the coupling gate is the
\emph{symmetric product} $g_{kj}^n = a_k^n a_j^n \in \{0,1\}$, so
$[P_n^{(r)}]_{kj} = -a_k^n a_j^n c_{kj,r}$ off the diagonal. Two properties
motivate this over the asymmetric alternative $g_{kj}^n = a_j^n$, which gates
only by the source modality: $P_n$ remains symmetric, which
Assumption~\ref{as:sym} and every result below require, and an unobserved
modality is removed from the graph entirely.

\begin{proposition}[Exact isolation of unobserved modalities]\label{prop:isolation}
If $a_j^n = 0$, then $P_n$ is block diagonal with respect to the splitting
$\{j\} \cup ([M]\setminus\{j\})$, hence so are $S_n$ and $\Top_n$ of
\eqref{eq:recurrence}. Consequently, at visit $n$ the latent state of modality
$j$ has exactly zero influence on every other modality and vice versa, while
modality $j$ continues to evolve under its own $A_j$ with
$S_n^{jj} = (I_d + \Delta t_n^2 A_j)^{-1}$. Under the asymmetric gate
$g_{kj} = a_j$, $P_n$ is block \emph{triangular} rather than block diagonal, so
the isolation is one-sided and $P_n$ is no longer symmetric, invalidating
Lemma~\ref{lem:psd}.
\end{proposition}

Without the gate, the one-step cross-modal influence is
$\bigl\lvert [S_n^{(r)}]_{kj} \bigr\rvert = \Delta t_n^{2}\, \lvert c_{kj,r} \rvert + O\bigl(\Delta t_n^{4} L_P^{2}\bigr)$
at short gaps, peaks at intermediate gaps and decays as $\Delta t_n^{-2}$ at long
ones (Proposition~\ref{prop:leak}), so the gate matters most in the intermediate
regime and for slow channels. The structural point is that $P_n$ sits
in the homogeneous part of the dynamics and is a function of $a^n$, and the
following theorem is the formal consequence.

\begin{definition}[Input-gated linear layers]\label{def:classA}
Let $\Acal$ be the class of layers
$x^n = T_n x^{n-1} + \sum_{k=1}^{M} a_k^n B_k^n \tilde u_k^n$ with $x^0 = 0$ and
$o^n = R x^n$, in which $T_n$, $B_k^n$ and $R$ may depend arbitrarily on the
parameters, on $\{t_m\}$ and on $\{\Delta t_m\}$, but \emph{not} on the
availability pattern $a^{1:N}$.
\end{definition}

$\Acal$ contains the state-space layers of S4, S5, LRU, LinOSS and D-LinOSS
applied to concatenated multimodal inputs with masked or zero-filled missing
entries. It excludes selective models such as Mamba when the availability vector
is appended to the input, since selection then makes $T_n$ depend on $a$.

\begin{theorem}[Availability additivity]\label{thm:additivity}
For every layer in $\Acal$ and every $n$,
$o^n = \sum_{m \le n} \sum_{k} a_k^m \psi_{k,m}$ with
$\psi_{k,m} := R\,T_{n}\cdots T_{m+1} B_k^m \tilde u_k^m$, so $o^n$ is a
multilinear polynomial of degree one in $a^{1:n}$ and every mixed availability
difference of order two or higher vanishes identically,
$\Delta_{(k,m)}\Delta_{(j,m')} o^n \equiv 0$ for all $(k,m)\neq(j,m')$, where
$\Delta_{(k,m)} F(a) := F(a\,|\,a_k^m{=}1) - F(a\,|\,a_k^m{=}0)$. CAMOS does not
lie in $\Acal$: already for $M=2$, $d=1$, $L=1$,
\begin{equation}
  [S_n]_{12}
  = \frac{\Delta t_n^2\, g\, c}
         {(1+\Delta t_n^2\alpha_1)(1+\Delta t_n^2\alpha_2)
          - \Delta t_n^4 g^2 c^2},
  \qquad g = a_1^n a_2^n,
  \label{eq:S12}
\end{equation}
which vanishes when $g=0$ and is nonzero when $g=1$ for any $c \neq 0$, giving
$\Delta_{(1,n)}\Delta_{(2,n)} o^n \neq 0$.
\end{theorem}

Theorem~\ref{thm:additivity} concerns a \emph{single linear layer}. It does not
say that deep networks built from such layers cannot represent availability
interactions, since a nonlinear readout applied to an additive statistic is not
itself additive, nor does it separate CAMOS from representation-level fusion
such as MedFuse or DrFuse, whose encoders are nonlinear in $a$. It says that in
class $\Acal$ the interaction must be reconstructed by the readout from an
additive summary, whereas CAMOS represents it inside the recurrence, in the part
of the model that is scannable and carries the stability guarantees above
(Remark~\ref{rem:scope}).

\subsection{Stable discretization on the irregular visit grid}
\label{sec:disc}

We discretize \eqref{eq:ode} on $\{t_n\}$ with the fully implicit (backward
Euler) scheme applied to the first-order form,
$z^n = z^{n-1} + \Delta t_n(-P_n y^n + f^n)$ and
$y^n = y^{n-1} + \Delta t_n z^n$. \label{eq:im} Collecting
$x^n := (z^n, y^n)^\top \in \Rb^{2m}$ and inverting through the Schur complement
$S_n := (I_m + \Delta t_n^2 P_n)^{-1}$, which exists for every
$\Delta t_n \in \Rb$ by Lemma~\ref{lem:psd}, gives
\begin{equation}
  x^n = \Top_n\, x^{n-1} + \Top_n F_n,
  \qquad
  \Top_n = \begin{pmatrix} S_n & -\Delta t_n P_n S_n \\
                        \Delta t_n S_n & S_n \end{pmatrix},
  \qquad F_n = (\Delta t_n f^n, 0)^\top .
  \label{eq:recurrence}
\end{equation}
Under \eqref{eq:channelblock}, $S_n$ is block diagonal with $d$ blocks
$S_n^{(r)} = (I_M + \Delta t_n^2 P_n^{(r)})^{-1}$, each costing $O(M^3)$ by
Cholesky. The step size enters as a per-visit scalar, so irregular sampling
requires no modification: the model integrates the same continuous system over
whatever gap actually occurred, and stability is uniform in that gap.

\begin{proposition}[Strict spectral contraction, uniform in $\Delta t_n$ and $a^n$]
\label{prop:spectral}
Let $\nu_1,\dots,\nu_m > 0$ be the eigenvalues of $P_n$. The $2m$ eigenvalues of
$\Top_n$ are
$\lambda_{i}^{\pm} = (1 \pm \mathrm{i}\Delta t_n\sqrt{\nu_i})/(1+\Delta t_n^2\nu_i)$,
so $|\lambda_i^{\pm}| = (1+\Delta t_n^2\nu_i)^{-1/2}
\le (1+\Delta t_n^2\mu)^{-1/2} < 1$ for every $\Delta t_n > 0$ and every
availability pattern.
\end{proposition}

The proof does not follow from the uncoupled case, because $P_n$ is not
diagonal; it orthogonally diagonalizes the symmetric $P_n$ and reduces to $m$
scalar $2\times2$ problems. Proposition~\ref{prop:spectral} bounds each
$\Top_n$ individually, which is not sufficient for boundedness of the
trajectory, since the $\Top_n$ vary with $n$ and are not normal, and finite
products of matrices with spectral radius below one can grow transiently. This
gap does not arise in the constant-coefficient setting, and we close it with an
energy argument.

\begin{theorem}[Switching bound governed by availability transitions]
\label{thm:switching}
Let $E_n := \tfrac12(\|z^n\|^2 + (y^n)^\top P_n y^n)$ with $f^n = 0$, and let
$V_N := \#\{\,n \le N : a^n \neq a^{n-1}\,\}$ count the visits at which the
availability pattern changes. Then
\begin{equation}
  E_N \le E_0 \,
  \exp\!\Bigl(\tfrac{1}{\mu}\textstyle\sum_{n=1}^{N}\|P_n - P_{n-1}\|_2\Bigr)
  \le E_0\, \exp\!\bigl(2\gamma V_N/\mu\bigr),
  \label{eq:switchbound}
\end{equation}
and $E_n$ is non-increasing on every maximal run of visits over which $a^n$ is
constant, with exact dissipation
$E_n - E_{n-1} = -\tfrac12\Delta t_n^2((y^n)^\top P^2 y^n + (z^n)^\top P z^n)$
when $P_n \equiv P$ (Proposition~\ref{prop:energy}).
\end{theorem}

The bound is length-independent: amplification is charged to \emph{availability
transitions}, not to the number of visits or the size of the gaps, and it makes
$\epsilon$ quantitative through
$\gamma/\mu \le \frac{1-\epsilon}{\epsilon}\,\alpha_{\max}/\alpha_{\min}$. The
choice of scheme is not free. The implicit-explicit symplectic discretization
used by uncoupled oscillatory SSMs keeps unit-modulus eigenvalues only while
$\Delta t_n^2\nu \le 4$ for every eigenvalue $\nu$ of $P_n$, and otherwise has a
real eigenvalue pair with product one, so one of them exceeds unity and the
recurrence diverges. In ADNI, gaps range over
roughly $[0.5, 4]$ years, so enforcing that condition uniformly would require
$L_P \le 0.25$, hence $\alpha_{\max} \lesssim 0.14$, an order of magnitude below
our initialization range and enough to collapse the representable frequency
band.

\subsection{Deep blocks, within-visit attention and task heads}
\label{sec:blocks}

\paragraph{GCO block.} Following the block design of oscillatory SSMs
\citep{rusch2025linoss}, block $\ell$ maps
$\{\tilde u^{(\ell-1),n}\}_{n=1}^N \subset \Rb^{Mh}$ to
$\{\tilde u^{(\ell),n}\}_{n=1}^N$ by scanning \eqref{eq:recurrence}, forming
$w^{(\ell),n} = C^{(\ell)}_{\mathrm{ro}} y^{(\ell),n}
+ D^{(\ell)}_{\mathrm{ro}} \tilde u^{(\ell-1),n}$ with
$C_{\mathrm{ro}} \in \Rb^{Mh\times m}$, and applying a residual GLU
\citep{dauphin2017glu} with dropout $0.1$
(Appendix~\ref{app:blocks}). Each block owns its own
$\hat\alpha, \hat c, B, C_{\mathrm{ro}}, D_{\mathrm{ro}}$, while the pattern
$a^{1:N}$ and the gaps $\{\Delta t_n\}$ are shared. We write
$h_k^n := \iota_k^\top y^{(L),n} \in \Rb^{d}$.

\paragraph{Within-visit cross-modal attention.} GCO layers model how modalities
drive one another \emph{over time}. Contemporaneous associations at a single
visit are captured by one multi-head attention block over the $M$ modality
tokens at each visit, applied in parallel across $n$
\citep{vaswani2017attention}, with an additive availability bias
$-\beta(1-a_j^n)$ on the logit of key $j$, where $\beta = \spls(\hat\beta) \ge 0$
is learnable. Taking $\beta \to \infty$ recovers the hard availability mask of
\citet{yao2024drfuse}; a finite $\beta$ down-weights but does not erase an
unobserved modality, consistent with its latent state continuing to evolve under
Proposition~\ref{prop:isolation}. We use $H = 4$ heads with a residual
connection and layer normalization, and write
$\bar h_k^n := h_k^n + \mathrm{Attn}(h^n)_k$ and
$\bar h^n := \Vert_{k} \bar h_k^n$.

\paragraph{Task heads and objective.} All heads are affine in $\bar h$: same-visit staging
is $\hat p^n = \mathrm{softmax}(W_{\mathrm{st}} \bar h^n + b_{\mathrm{st}})$;
landmark prediction is $\hat q = \sigma(w_{\mathrm{cv}}^\top \sum_{n} \omega_n \bar h^n
+ b_{\mathrm{cv}})$, pooled over visits by learned weights $\omega_n$ so that
subjects with different numbers of visits are handled without padding artefacts;
and longitudinal forecasting is
$\hat u_k^{n+h} = W_{\mathrm{fc},k}^{(h)} \bar h_k^n + b_{\mathrm{fc},k}^{(h)}$
for $h \le H_{\max} = 3$. The objective is
\begin{equation}
  \mathcal{L}
  = \lambda_{\mathrm{st}} \mathcal{L}_{\mathrm{st}}
  + \lambda_{\mathrm{cv}} \mathcal{L}_{\mathrm{cv}}
  + \lambda_{\mathrm{fc}} \mathcal{L}_{\mathrm{fc}}
  + \lambda_{\mathrm{reg}} \sum_{\ell,k,j,r} \bigl|\hat c^{(\ell)}_{kj,r}\bigr|,
  \label{eq:loss}
\end{equation}
with $\mathcal{L}_{\mathrm{st}}$ class-weighted cross-entropy over visits with a
recorded stage label, $\mathcal{L}_{\mathrm{cv}}$ class-weighted binary
cross-entropy, and $\mathcal{L}_{\mathrm{fc}}$ a horizon-discounted $\ell_1$
error masked by $a_k^{n+h}$ and normalized by the number of observed targets
(Appendix~\ref{app:blocks}). The mask is essential: with roughly half of all
modality-visit pairs unobserved, scoring against imputed targets would train the
model on its own imputation error. The forecasting term also acts as a physical
regularizer, forcing the latent oscillators to track measurable quantities, and
the $\ell_1$ penalty on the raw couplings encourages a sparse modality graph
without interfering with the budget \eqref{eq:param}. We use
$(\lambda_{\mathrm{st}}, \lambda_{\mathrm{cv}}, \lambda_{\mathrm{fc}},
\lambda_{\mathrm{reg}}) = (1, 0.5, 1, 10^{-3})$.

\begin{table}[t]
\caption{Same-visit CN/MCI/AD staging on ADNI and OASIS-3. $^{\dagger}$ denotes that CAMOS is significantly better than the best baseline ($p < 0.01$, two-sided Welch's $t$-test).}
\label{tab:staging}
\centering
\small
\setlength{\tabcolsep}{3pt}
\begin{tabular}{lccccc}
\toprule
Models & Accuracy \up & Macro F1 \up & Macro Prec.\ \up & Macro Recall \up & Macro Spec.\ \up \\
\midrule
\multicolumn{6}{l}{\textbf{ADNI}} \\
\cmidrule(r){1-1}
LinOSS-IM   & $0.6593_{\pm 0.0041}$ & $0.6815_{\pm 0.0039}$ & $0.6804_{\pm 0.0045}$ & $0.6864_{\pm 0.0037}$ & $0.8147_{\pm 0.0021}$ \\
LinOSS-IMEX & $0.7206_{\pm 0.0058}$ & $0.7234_{\pm 0.0061}$ & $0.7394_{\pm 0.0055}$ & $0.7424_{\pm 0.0059}$ & $0.8513_{\pm 0.0030}$ \\
MedFuse     & $0.6648_{\pm 0.0072}$ & $0.6697_{\pm 0.0076}$ & $0.6564_{\pm 0.0069}$ & $0.7139_{\pm 0.0070}$ & $0.8284_{\pm 0.0034}$ \\
DrFuse      & $0.7322_{\pm 0.0049}$ & $0.7354_{\pm 0.0046}$ & $0.7488_{\pm 0.0044}$ & $0.7683_{\pm 0.0042}$ & $0.8599_{\pm 0.0024}$ \\
\textbf{CAMOS} & $\best{0.7824}^{\dagger}_{\pm \best{0.0016}}$ & $\best{0.7790}^{\dagger}_{\pm \best{0.0018}}$ & $\best{0.7970}^{\dagger}_{\pm \best{0.0022}}$ & $\best{0.7786}^{\dagger}_{\pm \best{0.0025}}$ & $\best{0.8836}^{\dagger}_{\pm \best{0.0013}}$ \\
\midrule
\multicolumn{6}{l}{\textbf{OASIS-3}} \\
\cmidrule(r){1-1}
LinOSS-IM   & $0.7271_{\pm 0.0017}$ & $0.3253_{\pm 0.0053}$ & $0.5334_{\pm 0.0082}$ & $0.3541_{\pm 0.0025}$ & $0.6773_{\pm 0.0014}$ \\
LinOSS-IMEX & $0.7231_{\pm 0.0020}$ & $0.3027_{\pm 0.0064}$ & $0.4222_{\pm 0.0094}$ & $0.3414_{\pm 0.0028}$ & $0.6756_{\pm 0.0016}$ \\
MedFuse     & $0.6682_{\pm 0.0032}$ & $0.3176_{\pm 0.0059}$ & $0.2871_{\pm 0.0091}$ & $0.3558_{\pm 0.0026}$ & $0.6844_{\pm 0.0019}$ \\
DrFuse      & $0.7263_{\pm 0.0023}$ & $0.2805_{\pm 0.0081}$ & $0.2421_{\pm 0.0086}$ & $0.3333_{\pm 0.0000}$ & $0.6667_{\pm 0.0000}$ \\
\textbf{CAMOS} & $\best{0.7600}^{\dagger}_{\pm \best{0.0014}}$ & $\best{0.5557}^{\dagger}_{\pm \best{0.0028}}$ & $\best{0.6280}^{\dagger}_{\pm \best{0.0031}}$ & $\best{0.5261}^{\dagger}_{\pm \best{0.0030}}$ & $\best{0.7682}^{\dagger}_{\pm \best{0.0015}}$ \\
\bottomrule
\end{tabular}
\end{table}

\section{Experiments}
\label{sec:experiments}

\paragraph{Datasets.}
\label{sec:datasets}
We evaluate on ADNI \citep{petersen2010adni} with $M = 4$ modalities grouped by acquisition route: \emph{imaging} ($p_1 = 8$; intracranial-volume-normalized volumes of hippocampus, entorhinal cortex, fusiform gyrus, middle temporal gyrus, whole brain and ventricles, plus FDG and AV45 PET SUVRs), \emph{CSF} ($p_2 = 3$; A$\beta_{1\text{-}42}$, total tau, phosphorylated tau), \emph{genetics} ($p_3 = 3$; \textit{APOE}~$\varepsilon4$ allele count, genotype indicators, polygenic risk summary) and \emph{cognition} ($p_4 = 5$; MMSE, ADAS-Cog-13, CDR-SB, RAVLT immediate recall, FAQ), with baseline age, sex and years of education as static covariates ($p_s = 3$). We retain subjects with at least three diagnosed visits and one fully observed modality per visit, and mark a modality available only when all of its features are recorded. We also evaluate on OASIS-3 \citep{lamontagne2019oasis3}, which carries the three modalities (\emph{imaging}, \emph{genetics}, \emph{cognition}) and three targets, refitting only the normalization statistics.

\paragraph{Baselines.}
\label{sec:baselines}
LinOSS-IM and LinOSS-IMEX \citep{rusch2025linoss} are uncoupled oscillatory SSMs, the direct parents of CAMOS, and isolate the effect of coupling. Having no native mechanism for missing modalities, they receive concatenated features with unobserved entries zero-filled, the availability vector $a^n$ as $M$ extra channels, and $\Delta t_n$ both as an input channel and as the per-step discretization interval. MedFuse \citep{hayat2022medfuse} and DrFuse \citep{yao2024drfuse} handle missing modalities at the representation level; we replace their modality encoders by per-modality gated recurrent encoders over the visit sequence, so they receive the same temporal information as the SSMs. Appendix \ref{app:extended} contains results for all the different families of baselines.
\section{Results}
\label{sec:results}

\paragraph{Same-visit Staging.}
\label{sec:res-staging}
On ADNI, CAMOS outperforms all the baselines (Table~\ref{tab:staging}), improving over DrFuse, by $0.0502$ in accuracy ($+6.9\%$ relative) and $0.0436$ in macro F1 ($+5.9\%$). LinOSS-IMEX outperforms LinOSS-IM, since the implicit-explicit scheme is stable only while $\Delta t_n^2 \nu \le 4$. This suggests that the tuned baseline learned frequencies low enough to stay in the stable regime. On OASIS-3, evaluated zero-shot with only the normalization statistics refitted, accuracy is weakly informative because one class dominates. DrFuse attains macro recall of exactly $1/3$ and macro specificity of exactly $2/3$, the signature of predicting a single class, so its accuracy of $0.7263$ equals the majority-class rate. The other baselines are close to this solution, with macro F1 between $0.2805$ and $0.3253$. CAMOS alone retains classification among the three stages, raising macro F1 from $0.3253$ to $0.5557$ and macro recall from $0.3558$ to $0.5261$, while exceeding the majority-class accuracy. Since these relative gains are measured against near-collapsed baselines, the absolute macro F1 of CAMOS is the informative quantity for transfer.

\paragraph{Landmark Prediction.}
\label{sec:res-conversion}
Our CAMOS improves AUROC from $0.9386$ to $0.9611$ ($+2.4\%$) and AUPRC from $0.7689$ to $0.8687$ ($+13.0\%$) over DrFuse on ADNI dataset (Table~\ref{tab:conversion}). The asymmetry between the two metrics is the main finding. AUROC is compressed near its ceiling, spanning only $4.8$ points across the five models ($0.9134$ to $0.9611$), whereas AUPRC, which is sensitive to the minority positive class, spans $17.9$ points ($0.6900$ to $0.8687$). This matters practically because landmark prediction is used for trial enrichment, a precision-limited application in which AUPRC is the relevant summary. On OASIS-3 the pattern is sharper. AUROC barely separates CAMOS from the oscillatory baselines ($0.9600$ vs. $0.9560$ for LinOSS-IMEX, $+0.4\%$), while AUPRC rises from $0.4185$ to $0.7894$ ($+88.6\%$), with all baselines between $0.2594$ and $0.4185$.

\begin{table}[t]
\caption{3-year MCI-AD landmark prediction on ADNI and OASIS-3. CAMOS is significantly better than the best baseline at
$^{\dagger}p < 0.01$ or $^{*}p < 0.05$ (two-sided Welch's $t$-test).}
\label{tab:conversion}
\centering
\small
\setlength{\tabcolsep}{5pt}
\begin{tabular}{lcccc}
\toprule
& \multicolumn{2}{c}{\textbf{ADNI}} & \multicolumn{2}{c}{\textbf{OASIS-3}} \\
\cmidrule(lr){2-3}\cmidrule(lr){4-5}
Models & AUROC \up & AUPRC \up & AUROC \up & AUPRC \up \\
\midrule
LinOSS-IM
& $0.9223_{\pm 0.0031}$
& $0.7102_{\pm 0.0087}$
& $0.9533_{\pm 0.0024}$
& $0.3588_{\pm 0.0142}$ \\

LinOSS-IMEX
& $0.9307_{\pm 0.0027}$
& $0.7389_{\pm 0.0074}$
& $0.9560_{\pm 0.0029}$
& $0.4185_{\pm 0.0163}$ \\

MedFuse
& $0.9134_{\pm 0.0045}$
& $0.6900_{\pm 0.0102}$
& $0.8467_{\pm 0.0058}$
& $0.2594_{\pm 0.0121}$ \\

DrFuse
& $0.9386_{\pm 0.0022}$
& $0.7689_{\pm 0.0069}$
& $0.8787_{\pm 0.0047}$
& $0.3183_{\pm 0.0135}$ \\

\textbf{CAMOS}
& $\best{0.9611}^{\dagger}_{\pm \best{0.0014}}$
& $\best{0.8687}^{\dagger}_{\pm \best{0.0041}}$
& $\best{0.9600}^{*}_{\pm \best{0.0018}}$
& $\best{0.7894}^{\dagger}_{\pm \best{0.0093}}$ \\
\bottomrule
\end{tabular}
\end{table}

\paragraph{Longitudinal forecasting.}
\label{sec:res-forecast}
Figure~\ref{fig:forecast} reports next-visit ($h=1$) error per modality. On ADNI, CAMOS attains the lowest MAE and RMSE on all four modalities. The largest reductions are on the sparsely observed genetics (MAE $0.7535 \to 0.4088$, $-45.8\%$; RMSE $-38.3\%$) and CSF (MAE $0.7676 \to 0.5906$, $-23.1\%$; RMSE $-18.3\%$), where the four baselines are tightly clustered (MAE $0.7535$ to $0.8270$ and $0.7676$ to $0.7895$), so uncoupled oscillatory dynamics and representation-level fusion perform alike there while CAMOS separates clearly. Gains are smaller on cognition (MAE $-14.1\%$, RMSE $-7.0\%$) and imaging (MAE $-7.2\%$, RMSE $-11.8\%$), the only modality on which the baselines split, with LinOSS-IM and MedFuse (MAE $0.4468$, $0.4538$) ahead of LinOSS-IMEX and DrFuse ($0.5567$, $0.5562$). The same ordering holds zero-shot on OASIS-3: CAMOS reduces MAE most on genetics ($0.7875 \to 0.2847$, $-63.9\%$), then on clinical measures ($0.5775 \to 0.3756$, $-35.0\%$), and least on imaging ($0.7664 \to 0.7333$, $-4.3\%$).

\begin{figure}[t]
\centering
\begin{subfigure}[t]{0.5\textwidth}
    \centering
    \includegraphics[width=\linewidth]{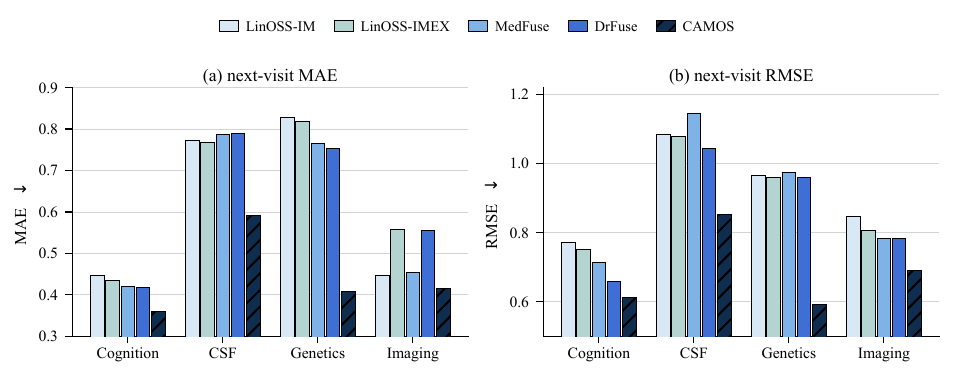}
\end{subfigure}
\hfill
\begin{subfigure}[t]{0.49\textwidth}
    \centering
    \includegraphics[width=\linewidth]{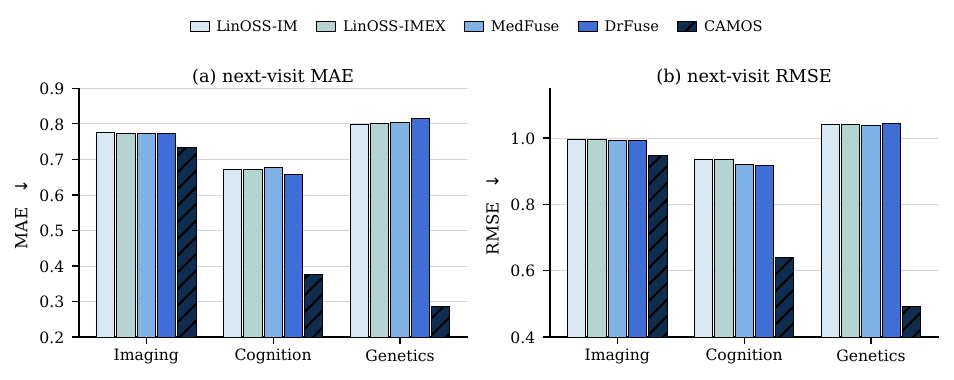}
\end{subfigure}
\caption{Next-visit ($h=1$) forecasting on ADNI (left) and OASIS (right) per modality.}
\label{fig:forecast}
\end{figure}

\paragraph{Relative improvements over baselines.}
All relative improvements in Figure~\ref{fig:combined} are positive, ranging from $+1.3\%$ (macro recall) to $+45.8\%$ (genetics MAE) on ADNI and from $+0.4\%$ (AUROC) to $+88.6\%$ (AUPRC) on OASIS-3, where margins are largest on metrics for which baselines collapse to the majority class (F1 $+70.8\%$, recall $+47.9\%$) and smallest on the imbalance-saturated accuracy and AUROC. The ADNI ablations trace these gains to three design choices. The implicit discretization, stable for every gap, matters most: switching to the implicit-explicit scheme, unstable once $\Delta t_n^2 \nu > 4$ as the $0.5$ to $4$ year gaps allow, lowers macro F1 from $0.7790$ to $0.653$ and AUPRC from $0.8687$ to $0.8063$, and raises forecasting MAE by $39.8\%$ (sparse) and $25.5\%$ (dense). Coupling lets each modality's state absorb cross-modal information at co-observed visits before an unobserved interval; removing it raises MAE by $7.7\%$ (sparse) and $8.1\%$ (dense) and lowers AUPRC to $0.8193$, the two regions of largest margin. The availability gate acts selectively: removing it lowers AUROC to $0.9531$ and AUPRC to $0.8296$ but leaves staging (F1 $0.7804$) and forecasting (MAE $0.5016$ sparse, $0.4014$ dense) essentially unchanged. It also explains the transfer result, since the CSF oscillator is not observed in OASIS-3 and stays isolated instead of entering as a zero-filled input block.

\begin{figure}[t]
\centering
\begin{subfigure}[t]{0.5\textwidth}
    \centering
    \includegraphics[width=\linewidth]{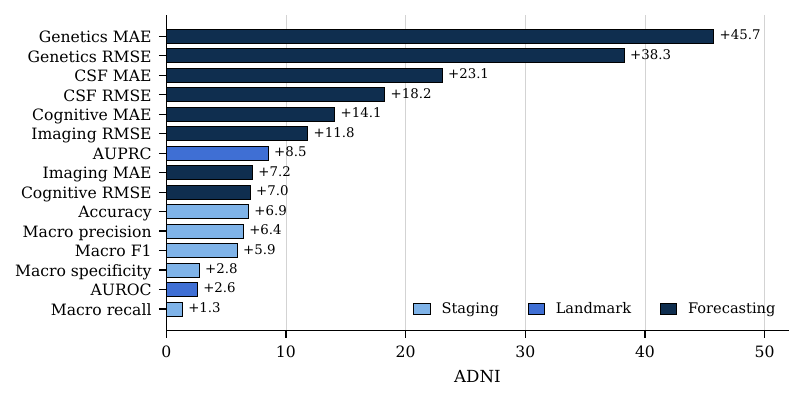}
\end{subfigure}
\hfill
\begin{subfigure}[t]{0.49\textwidth}
    \centering
    \includegraphics[width=\linewidth]{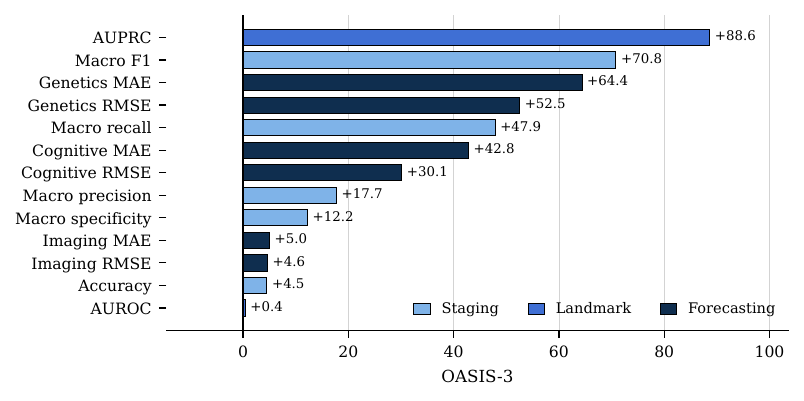}
\end{subfigure}
\caption{Relative improvement of CAMOS over the best baseline on each of
    the reported metrics.}
\label{fig:combined}
\end{figure}

\paragraph{Ablation studies.}
\label{sec:ablation}
Table~\ref{tab:ablation} evaluates different variants of CAMOS on ADNI. Removing cross-modal coupling ($C \equiv 0$) degrades every task, lowering staging macro F1 by $2.3$ points and landmark prediction AUPRC by $4.9$ points while increasing forecasting MAE by $7.7\%$ (sparse) and $8.1\%$ (dense), which shows that information flow between modalities is beneficial. Coupling alone is not sufficient, without gating ($g_{kj}^{n} \equiv 1$), staging improves slightly over the uncoupled variant, but landmark prediction and forecasting become worse than having no coupling at all (AUPRC $0.7916$, MAE $+11.2\%$ sparse and $+11.4\%$ dense), plausibly because unobserved or unreliable states propagate indiscriminately across modalities. An asymmetric gate ($g_{kj}^{n}=a_j^n$) recovers much of this loss and is the strongest ablation on landmark prediction and forecasting, yet it yields the lowest staging F1 among the non-IMEX variants ($0.7532$), so only the full gating performs well on all tasks at once. Replacing the implicit discretization with an IMEX scheme causes the largest degradation, increasing forecasting MAE by $39.8\%$ on sparse and $25.5\%$ on dense modalities; the stronger effect on sparse modalities is consistent with their irregular observation gaps.

\begin{table}[t]
\caption{Ablation study evaluating the contribution of coupling, gating, and discretisation choices.}
\label{tab:ablation}
\centering
\small
\setlength{\tabcolsep}{4pt}
\resizebox{\textwidth}{!}{%
\begin{tabular}{@{}p{0.30\textwidth}ccccc@{}}
\toprule
& \multicolumn{1}{c}{Staging} & \multicolumn{2}{c}{Landmark Prediction}
& \multicolumn{2}{c}{Longitudinal Forecasting} \\
\cmidrule(lr){2-2}\cmidrule(lr){3-4}\cmidrule(lr){5-6}
Variants & Macro F1 \up & AUROC \up & AUPRC \up
& Sparse MAE \dn & Dense MAE \dn \\
\midrule
\textit{w/o} coupling ($C \equiv 0$)
& $0.7561_{\pm 0.0034}$ & $0.9530_{\pm 0.0021}$ & $0.8193_{\pm 0.0068}$ & $0.5491_{\pm 0.0057}$ & $0.4354_{\pm 0.0032}$ \\
\textit{w} coupling ungated ($g_{kj}^{n} \equiv 1$)
& $0.7694_{\pm 0.0029}$ & $0.9478_{\pm 0.0026}$ & $0.7916_{\pm 0.0081}$ & $0.5668_{\pm 0.0063}$ & $0.4487_{\pm 0.0038}$ \\
\textit{w} asymmetric gate ($g_{kj}^{n}=a_j^n$)
& $0.7532_{\pm 0.0041}$ & $0.9562_{\pm 0.0019}$ & $0.8384_{\pm 0.0059}$ & $0.5315_{\pm 0.0049}$ & $0.4176_{\pm 0.0027}$ \\
\textit{w} IMEX discretization
& $0.6530_{\pm 0.0062}$ & $0.9546_{\pm 0.0024}$ & $0.8063_{\pm 0.0074}$ & $0.7127_{\pm 0.0089}$ & $0.5057_{\pm 0.0046}$ \\
\midrule
\textbf{CAMOS}
& $\best{0.7790}_{\pm \best{0.0018}}$
& $\best{0.9611}_{\pm \best{0.0014}}$
& $\best{0.8687}_{\pm \best{0.0041}}$
& $\best{0.5097}_{\pm \best{0.0036}}$
& $\best{0.4029}_{\pm \best{0.0022}}$ \\
\bottomrule
\end{tabular}%
}
\end{table}

\section{Conclusion}
In this paper, we argue that missingness in multimodal clinical sequences belongs in the transition operator rather than only in the forcing term. We propose CAMOS, a coupled oscillatory state-space model whose transition operator is gated by availability. Making coupling admissible required three new results: a per-channel Gershgorin budget that keeps the effective stiffness positive definite over all $2^M$ availability patterns and all gaps, an energy bound that charges amplification to availability transitions rather than sequence length, and a channel-block matrix family that keeps the associative scan exact at $O(M^2)$ overhead. On ADNI, CAMOS outperforms uncoupled oscillatory SSMs and clinical fusion models on same-visit staging, landmark prediction and longitudinal forecasting. On zero-shot OASIS-3, it is the only model that avoids collapse to the majority class. The ablations attribute these gains to the implicit discretization and the coupling.


\subsection*{AI use statement}

Generative AI tools were used for language editing of the manuscript. They were not used to generate experimental results, numerical values,
or the statements and proofs of the theoretical results, all of which were
derived by the authors. The authors take
full responsibility for the final content of this work, including all text,
claims and artifacts.

\subsection*{Ethics statement}

CAMOS predicts conversion to Alzheimer's disease and assigns clinical stages, outputs that could influence trial enrolment or a person's understanding of their own prognosis. We regard it as a research instrument rather than a decision-support tool. The results here are obtained on volunteer research cohorts that are not representative of the general population. The predicted probabilities are not calibrated, and the availability gate makes the model sensitive to which tests were ordered, a variable that reflects access to care as much as clinical need. Subgroup evaluation, calibration, and an audit of performance conditional on the availability pattern are therefore prerequisites for any clinical claim. All data are used under the existing data use agreements of ADNI and OASIS-3, and no new human-subject data were collected. The authors declare no competing interests.

\subsection*{Reproducibility statement}

The model is fully specified: 
Table~\ref{tab:hparams} for every hyperparameter and optimizer setting,
Appendix~\ref{app:algo} the algorithm in pseudocode, Appendix~\ref{app:impl} the
complete configuration table, and Appendix~\ref{app:data} the exact feature
definitions, inclusion criteria, preprocessing and label construction. The
assumptions underlying the theory are isolated as Assumptions~\ref{as:sym} and
\ref{as:budget}, and every proof appears in Appendix~\ref{app:proofs}.
An anonymized code repository accompanies this
submission, containing the model implementation, the cohort extraction and
preprocessing scripts, the training and evaluation pipeline for CAMOS and all
baselines. All datasets are available to qualified researchers from their original providers.

\bibliography{iclr2027_conference}
\bibliographystyle{iclr2027_conference}

\appendix
\newpage
\section{Appendix: Notation}
\label{app:notation}

\begin{table}[h]
\centering
\footnotesize
\setlength{\tabcolsep}{3pt}
\begin{tabular}{@{}ll@{\hspace{5mm}}ll@{}}
\toprule
Symbol & Meaning & Symbol & Meaning \\
\midrule
$M$ & number of modalities ($M=4$)          & $d$ & oscillators per modality \\
$N$ & visits for a subject                  & $m = Md$ & total latent width \\
$h$ & shared projection width               & $L$ & number of GCO blocks \\
$p_k$ & raw dimension of modality $k$       & $p_s$ & static covariate dimension \\
$t_n$ & time of visit $n$ (years)           & $\Delta t_n$ & $t_n - t_{n-1} > 0$ \\
$u_k^n \in \Rb^{p_k}$ & raw features        & $\tilde u_k^n \in \Rb^{h}$ & projected \\
$a_k^n \in \{0,1\}$ & availability indicator & $a^n \in \{0,1\}^M$ & availability pattern \\
$D_a$ & $\blkdiag(a_1 I_d,\dots,a_M I_d)$   & $g_{kj}^n = a_k^n a_j^n$ & coupling gate \\
$y(t), z(t) \in \Rb^m$ & position, velocity & $x = (z,y)^\top$ & full state, $\Rb^{2m}$ \\
$A_k = \diag(\alpha_{k,\cdot})$ & natural freqs. & $C_{kj} = \diag(c_{kj,\cdot})$ & coupling \\
$P = A - C$ & effective stiffness           & $P_n = A - D_{a^n}CD_{a^n}$ & gated stiffness \\
$P^{(r)} \in \Rb^{M\times M}$ & channel-$r$ block & $\Pi$ & channel permutation \\
$S_n = (I{+}\Delta t_n^2 P_n)^{-1}$ & inverse Schur cmpl. & $\Top_n$ & transition operator \\
$\epsilon$ & Gershgorin budget ($0.2$)      & $\delta$ & numerical floor ($10^{-6}$) \\
$\mu = \epsilon\alpha_{\min}$ & stiffness lower bound & $L_P = (2{-}\epsilon)\alpha_{\max}$ & upper bound \\
$\gamma = \max_r\|C^{(r)}\|_2$ & coupling strength & $H_{\max}$ & forecast horizons \\
$E_n$ & discrete energy                     & $V_N$ & availability transitions \\
$\Fcal$ & channel-block matrix family       & $\Acal$ & input-gated linear layers \\
\bottomrule
\end{tabular}
\end{table}

\section{Model Component Details}
\label{app:blocks}

\paragraph{Static encoder.} The initial condition of Section~\ref{sec:encoders}
is
\begin{equation}
  x^0 = (z^0, y^0)^\top
  = W_2^{s}\,\mathrm{GELU}\!\left(W_1^{s} s + b_1^{s}\right) + b_2^{s}
  \in \Rb^{2m},
  \qquad W_1^{s} \in \Rb^{2m \times p_s},\ W_2^{s} \in \Rb^{2m \times 2m}.
  \label{eq:static}
\end{equation}

\paragraph{GCO block.} Block $\ell$ maps
$\{\tilde u^{(\ell-1),n}\}_{n=1}^N$ to $\{\tilde u^{(\ell),n}\}_{n=1}^N$ by
\begin{align}
  \{x^{(\ell),n}\}_n
    &= \textsc{Scan}_\bullet\bigl(\{\Top_n^{(\ell)}\}_n,
        \{\Top_n^{(\ell)}F_n^{(\ell)}\}_n\bigr),
    \label{eq:block1}\\
  w^{(\ell),n} &= C^{(\ell)}_{\mathrm{ro}}\, y^{(\ell),n}
                 + D^{(\ell)}_{\mathrm{ro}}\, \tilde u^{(\ell-1),n},
    \qquad C_{\mathrm{ro}} \in \Rb^{Mh\times m},
    \label{eq:block2}\\
  \tilde u^{(\ell),n} &= \tilde u^{(\ell-1),n}
      + \mathrm{Drop}\Bigl(\mathrm{GLU}\bigl(\mathrm{GELU}(w^{(\ell),n})\bigr)\Bigr),
    \label{eq:block3}
\end{align}
with $\mathrm{GLU}(v) = \sigma(W_1 v)\odot W_2 v$ \citep{dauphin2017glu} and
dropout $0.1$.

\paragraph{Within-visit attention.} With $H$ heads and $d_H = d/H$,
\begin{equation}
  \mathrm{Attn}(h^n)_k
  = W_O \Bigl\Vert_{\eta=1}^{H}
    \sum_{j=1}^{M}
    \frac{\exp\bigl(
       \langle W_Q^{(\eta)\top} h_k^n, W_K^{(\eta)\top} h_j^n\rangle / \sqrt{d_H}
       - \beta(1-a_j^n)\bigr)}
      {\sum_{l=1}^{M}\exp\bigl(
       \langle W_Q^{(\eta)\top} h_k^n, W_K^{(\eta)\top} h_l^n\rangle / \sqrt{d_H}
       - \beta(1-a_l^n)\bigr)}
    \; W_V^{(\eta)\top} h_j^n ,
  \label{eq:attn}
\end{equation}
with $\beta = \spls(\hat\beta) \ge 0$ learnable and initialized at $\beta_0 = 2$.

\paragraph{Task heads and forecasting loss.} The three heads are
\begin{align}
  \hat p^n &= \mathrm{softmax}\bigl(W_{\mathrm{st}} \bar h^n + b_{\mathrm{st}}\bigr),
  &&W_{\mathrm{st}} \in \Rb^{3 \times m},
  \label{eq:head-st}\\
  \hat q &= \sigma\Bigl(w_{\mathrm{cv}}^\top
      \textstyle\sum_{n} \omega_n \bar h^n + b_{\mathrm{cv}}\Bigr),
  &&\omega_n = \tfrac{\exp(v^\top \bar h^n)}{\sum_{n'}\exp(v^\top \bar h^{n'})},
  \label{eq:head-cv}\\
  \hat u_k^{n+h} &= W_{\mathrm{fc},k}^{(h)}\, \bar h_k^n + b_{\mathrm{fc},k}^{(h)},
  &&h \in \{1,\dots,H_{\max}\},
  \label{eq:head-fc}
\end{align}
and the masked forecasting term of \eqref{eq:loss} is
\begin{equation}
  \mathcal{L}_{\mathrm{fc}}
    = \frac{\sum_{h=1}^{H_{\max}} \gamma_h
             \sum_{n,k} a_k^{n+h}\,
             \bigl\| \hat u_k^{n+h} - u_k^{n+h} \bigr\|_1}
            {\sum_{h=1}^{H_{\max}} \gamma_h \sum_{n,k} a_k^{n+h}\, p_k},
  \qquad \gamma_h = 2^{-(h-1)} .
  \label{eq:loss-fc}
\end{equation}

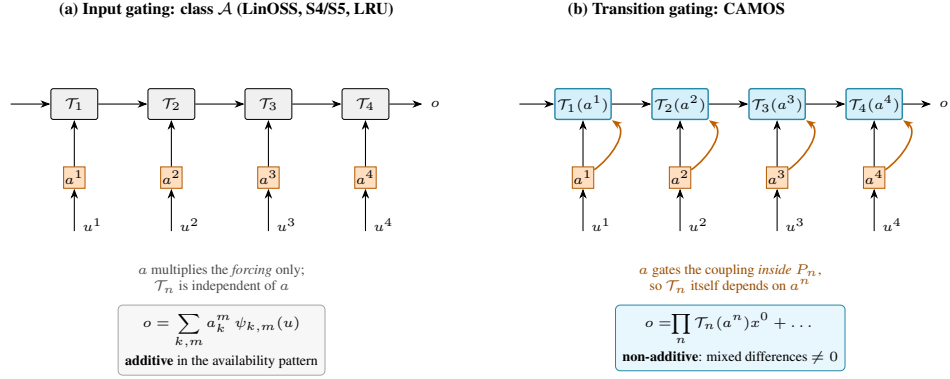
\begin{figure}[h]
\centering
\resizebox{0.9\textwidth}{!}{\input{figures/fig_separation.tex}}
\caption{Input gating versus transition gating. In (a) the indicators multiply
the forcing only, so the state unrolls into a sum of per-modality, per-visit
contributions and every mixed availability difference vanishes. In (b) the gate
enters the homogeneous part, so the state is a product of availability-dependent
operators and contains genuine interactions among availability indicators.}
\label{fig:separation}
\end{figure}

\section{Proofs}
\label{app:proofs}

\subsection{Derivation of the implicit recurrence \eqref{eq:recurrence}}

Scheme \eqref{eq:im} reads $\Mop_n x^n = x^{n-1} + F_n$ with
$\Mop_n = \bigl(\begin{smallmatrix} I & \Delta t_n P_n \\ -\Delta t_n I & I
\end{smallmatrix}\bigr)$. Direct multiplication gives
\begin{equation}
  \begin{pmatrix} I & \Delta t_n P_n \\ -\Delta t_n I & I \end{pmatrix}
  \begin{pmatrix} S_n & -\Delta t_n P_n S_n \\ \Delta t_n S_n & S_n\end{pmatrix}
  =
  \begin{pmatrix}
    (I + \Delta t_n^2 P_n) S_n & 0 \\
    0 & (I + \Delta t_n^2 P_n) S_n
  \end{pmatrix}
  = I_{2m},
\end{equation}
since $(I + \Delta t_n^2 P_n)S_n = I_m$ by definition of $S_n$. Hence
$\Top_n = \Mop_n^{-1}$ has the stated form. \hfill$\square$

\subsection{Proof of Lemma~\ref{lem:psd}}

Fix $a \in [0,1]^M$ and a channel $r$. By Assumption~\ref{as:sym} the matrix
$P^{(r)}(a) \in \Rb^{M\times M}$ with $[P^{(r)}(a)]_{kk} = \alpha_{k,r}$ and
$[P^{(r)}(a)]_{kj} = -a_k a_j c_{kj,r}$ for $k \neq j$ is symmetric, so its
eigenvalues are real. By the Gershgorin circle theorem
\citep[Thm.~6.1.1]{horn2012matrix} every eigenvalue $\nu$ satisfies
$|\nu - \alpha_{k,r}| \le R_k$ for some $k$, with
\begin{equation}
  R_k = \sum_{j\neq k} \bigl|a_k a_j c_{kj,r}\bigr|
      \le \sum_{j\neq k} \bigl|c_{kj,r}\bigr|
      \le (1-\epsilon)\,\alpha_{k,r},
\end{equation}
the first inequality using $a_k, a_j \in [0,1]$ and the second
Assumption~\ref{as:budget}. Hence
$\nu \in [\epsilon\alpha_{k,r}, (2-\epsilon)\alpha_{k,r}]$ for some $k$, so
$\nu \ge \mu$ and $\nu \le L_P$. Since
$\Pi P(a)\Pi^\top = \blkdiag(P^{(1)}(a),\dots,P^{(d)}(a))$ and permutation
conjugation preserves the spectrum, the same bounds hold for $P(a)$. The
eigenvalues of $I_m + \Delta t^2 P(a)$ are
$1 + \Delta t^2\nu \ge 1 + \Delta t^2\mu > 0$, so it is symmetric positive
definite and admits a Cholesky factorization. \hfill$\square$

\subsection{Proof of Proposition~\ref{prop:isolation}}

If $a_j^n = 0$ then for every $k \neq j$ the $(k,j)$ and $(j,k)$ blocks of
$D_{a^n}CD_{a^n}$ are $a_k^n a_j^n C_{kj} = 0$ and $a_j^n a_k^n C_{jk} = 0$.
Since $A$ is block diagonal, $P_n$ is block diagonal with respect to the
splitting $\{j\} \cup ([M]\setminus\{j\})$. The inverse of a block diagonal
matrix is block diagonal, so $S_n$ is block diagonal, and each of the four
blocks of $\Top_n$ in \eqref{eq:recurrence} is a product of $S_n$, $P_n$ and
scalars, hence block diagonal. Restricting to the $j$ block gives
$P_n|_{jj} = A_j$ and $S_n|_{jj} = (I_d + \Delta t_n^2 A_j)^{-1}$. Under the
asymmetric gate $g_{kj} = a_j$, the $(k,j)$ blocks vanish for all $k\neq j$ but
the $(j,k)$ blocks equal $-a_k^n C_{jk}$, which need not vanish, so $P_n$ is
block triangular and in general non-symmetric. \hfill$\square$

\subsection{Proof of Theorem~\ref{thm:additivity}}

Let $T_{n:m+1} := T_n T_{n-1}\cdots T_{m+1}$ with $T_{n:n+1} := I$. Unrolling
the recurrence of Definition~\ref{def:classA} from $x^0 = 0$,
\begin{equation}
  x^n = \sum_{m=1}^{n} T_{n:m+1} \sum_{k=1}^M a_k^m B_k^m \tilde u_k^m
      = \sum_{m=1}^n \sum_{k=1}^M a_k^m \, T_{n:m+1}B_k^m \tilde u_k^m ,
\end{equation}
and $o^n = Rx^n$ gives the additive form. The map $a^{1:n} \mapsto o^n$ is
affine in each coordinate with no products of distinct coordinates, so for any
$(k,m)\neq(j,m')$,
$\Delta_{(k,m)}\Delta_{(j,m')} o^n
= (\psi_{k,m} + \psi_{j,m'}) - \psi_{k,m} - \psi_{j,m'} = 0$.
For the second claim, take $M = 2$, $d = 1$, $L = 1$ and consider the $(1,2)$
entry of $S_n = (I_2 + \Delta t_n^2 P_n)^{-1}$ with
$P_n = \bigl(\begin{smallmatrix}\alpha_1 & -gc \\ -gc &
\alpha_2\end{smallmatrix}\bigr)$, $g = a_1^n a_2^n$, $c \neq 0$. The
$2\times2$ inverse gives \eqref{eq:S12}, which vanishes iff $g = 0$. Choosing an
input supported only on modality $1$ at visit $n$ and reading out coordinate $2$
of $y^n$, $o^n$ is a nonzero multiple of $[S_n]_{12}$ when $a_1^n = a_2^n = 1$
and zero if either indicator is zero, so
$\Delta_{(1,n)}\Delta_{(2,n)} o^n \neq 0$ and CAMOS is not in $\Acal$.
\hfill$\square$

\begin{remark}[Scope of Theorem~\ref{thm:additivity}]\label{rem:scope}
The theorem concerns a single \emph{linear state-space layer}. It does not say
that deep networks built from such layers cannot represent availability
interactions, since a nonlinear readout applied to an additive statistic is not
itself additive; it says that in class $\Acal$ the interaction must be
reconstructed by the readout from an additive summary, whereas CAMOS represents
it inside the recurrence, in the part of the model that is scannable and carries
the guarantees of Lemma~\ref{lem:psd} and Theorem~\ref{thm:switching}. Nor does
it separate CAMOS from representation-level fusion such as MedFuse or DrFuse,
whose encoders are nonlinear in $a$. We state it narrowly on purpose.
\end{remark}

\subsection{Proof of Proposition~\ref{prop:spectral}}

Since $P_n$ is symmetric, write $P_n = U\Lambda U^\top$ with $U$ orthogonal and
$\Lambda = \diag(\nu_1,\dots,\nu_m)$, $\nu_i \ge \mu > 0$ by
Lemma~\ref{lem:psd}. Put $\mathcal{U} = \blkdiag(U,U)$, which is orthogonal.
Then $\mathcal{U}^\top \Mop_n \mathcal{U} = \bigl(\begin{smallmatrix} I &
\Delta t_n \Lambda \\ -\Delta t_n I & I\end{smallmatrix}\bigr)$, which decouples
into $m$ independent $2\times2$ problems
$\bigl(\begin{smallmatrix} 1 & \Delta t_n \nu_i \\ -\Delta t_n &
1\end{smallmatrix}\bigr)$, each with inverse
$s_i\bigl(\begin{smallmatrix} 1 & -\Delta t_n \nu_i \\ \Delta t_n &
1\end{smallmatrix}\bigr)$, $s_i = (1+\Delta t_n^2\nu_i)^{-1}$. The eigenvalues
of this inverse are $s_i(1 \pm \mathrm{i}\Delta t_n\sqrt{\nu_i})$, so
$|\lambda_i^{\pm}|^2 = s_i^2(1 + \Delta t_n^2\nu_i) = s_i \le
(1+\Delta t_n^2\mu)^{-1} < 1$. Since $\Top_n = \Mop_n^{-1}$ and orthogonal
conjugation preserves the spectrum, the claim follows. The argument uses only
symmetry and positivity of $P_n$, not diagonality, which is why the uncoupled
proof does not transfer. \hfill$\square$

\begin{proposition}[Exact dissipation at fixed availability]\label{prop:energy}
Suppose $f^n = 0$ and $P_n \equiv P$. Then
$E_n := \tfrac12(\|z^n\|^2 + (y^n)^\top P y^n)$ satisfies, for every
$\Delta t_n > 0$,
\begin{equation}
  E_n - E_{n-1}
  = -\tfrac12 \Delta t_n^2 \bigl( (y^n)^\top P^2 y^n + (z^n)^\top P z^n \bigr)
  \le -\tfrac{\mu}{2}\Delta t_n^2\bigl(\mu\|y^n\|^2 + \|z^n\|^2\bigr) \le 0 .
\end{equation}
\end{proposition}

\begin{proof}
From \eqref{eq:im}, $z^{n-1} = z^n + \Delta t_n P y^n$ and
$y^{n-1} = y^n - \Delta t_n z^n$, so
\begin{align}
  \|z^{n-1}\|^2 &= \|z^n\|^2 + 2\Delta t_n \langle z^n, Py^n\rangle
                   + \Delta t_n^2 (y^n)^\top P^2 y^n, \\
  (y^{n-1})^\top P y^{n-1} &= (y^n)^\top P y^n
                   - 2\Delta t_n \langle z^n, Py^n\rangle
                   + \Delta t_n^2 (z^n)^\top P z^n .
\end{align}
Adding and halving, the cross terms cancel and
$E_{n-1} = E_n + \tfrac12\Delta t_n^2((y^n)^\top P^2 y^n + (z^n)^\top P z^n)$.
The inequality uses $P \succeq \mu I$ and hence $P^2 \succeq \mu^2 I$.
\end{proof}

The dissipation is negligible for the slowest mode at clinical sequence lengths,
since $\prod_{n}(1+\Delta t_n^2\mu)^{-1/2} \ge
\exp(-\tfrac{\mu}{2}\sum_n \Delta t_n^2)$ exceeds $0.90$ for $N \approx 10$,
$\Delta t_n \approx 1$ year and $\mu \approx 0.02$; high-frequency modes are
damped more strongly at long gaps, which is the price of unconditional
stability.

\subsection{Proof of Theorem~\ref{thm:switching}}

Write $E_n^{(Q)} := \tfrac12(\|z^n\|^2 + (y^n)^\top Q y^n)$, so
$E_n = E_n^{(P_n)}$. Applying Proposition~\ref{prop:energy} with $P_n$ frozen
over the step from $n-1$ to $n$ gives $E_n^{(P_n)} \le E_{n-1}^{(P_n)}$, and
\begin{equation}
  E_{n-1}^{(P_n)} = E_{n-1}^{(P_{n-1})}
    + \tfrac12 (y^{n-1})^\top (P_n - P_{n-1}) y^{n-1}
  \le E_{n-1}^{(P_{n-1})} + \tfrac12\|P_n - P_{n-1}\|_2\,\|y^{n-1}\|^2 .
\end{equation}
By Lemma~\ref{lem:psd}, $\|y^{n-1}\|^2 \le (2/\mu)E_{n-1}^{(P_{n-1})}$, so
$E_n \le E_{n-1}\exp(\tfrac{1}{\mu}\|P_n - P_{n-1}\|_2)$, and telescoping gives
the first inequality. For the second, if $a^n = a^{n-1}$ then $P_n = P_{n-1}$
and the factor is one, which also proves monotonicity; if $a^n \neq a^{n-1}$ then
\begin{equation}
  \|P_n - P_{n-1}\|_2
  \le \|D_{a^{n-1}}CD_{a^{n-1}}\|_2 + \|D_{a^n}CD_{a^n}\|_2 \le 2\gamma,
\end{equation}
using $\|D_a\|_2 \le 1$ and $\|C\|_2 = \max_r\|C^{(r)}\|_2 = \gamma$. Summing
over the $V_N$ transition steps completes the proof. \hfill$\square$

\begin{proposition}[Input-to-state bound]\label{prop:iss}
With forcing $f^n$ and $P_n \equiv P$,
$\sqrt{E_n} \le \sqrt{E_{n-1}} + \sqrt{2}\,\Delta t_n\|f^n\|$.
\end{proposition}

\begin{proof}
Repeating the computation of Proposition~\ref{prop:energy} with
$z^{n-1} = z^n + \Delta t_n(P y^n - f^n)$ gives
$2E_n \le 2E_{n-1} + 2\Delta t_n\|z^n\|\|f^n\|$. With
$A_\ast = \sqrt{2E_n}$, $B_\ast = \sqrt{2E_{n-1}}$, $c = \Delta t_n\|f^n\|$ and
$\|z^n\| \le A_\ast$, we get $A_\ast^2 \le B_\ast^2 + 2cA_\ast$, hence
$A_\ast \le c + \sqrt{c^2 + B_\ast^2} \le B_\ast + 2c$. Dividing by $\sqrt{2}$
gives the claim.
\end{proof}

\begin{proposition}[Cross-modal leakage and its dependence on the gap]\label{prop:leak}
Fix a channel $r$, an availability vector $a \in [0,1]^M$ (the ungated case is
$a = \mathbf{1}$), and $k \neq j$. Write $s := \Delta t_n^2$ and
$S_n^{(r)} = (I_M + s P_n^{(r)})^{-1}$ with
$P_n^{(r)} = \operatorname{diag}(\alpha_{\cdot,r}) - D_a C^{(r)} D_a$.
\begin{enumerate}[label=(\roman*)]
  \item \emph{Short gaps.}
  $\bigl|[S_n^{(r)}]_{kj} - s\, a_k a_j\, c_{kj,r}\bigr| \le s^2 L_P^2$,
  so the leakage grows quadratically in $\Delta t_n$ while $s L_P \ll 1$.
  \item \emph{Uniform bound.} For every $\Delta t_n > 0$,
  \begin{equation}\label{eq:leak-bound}
    \bigl|[S_n^{(r)}]_{kj}\bigr| \le \beta(s)
    := \frac12\Bigl(\frac{1}{1 + s\mu} - \frac{1}{1 + s L_P}\Bigr)
    = \frac{s\,(L_P - \mu)}{2(1 + s\mu)(1 + s L_P)} .
  \end{equation}
  $\beta$ increases for $s < (\mu L_P)^{-1/2}$ and decreases afterwards, attains
  its maximum $\tfrac12(\sqrt{L_P} - \sqrt{\mu})/(\sqrt{L_P} + \sqrt{\mu}) < \tfrac12$,
  and satisfies $\beta(s) \le (L_P - \mu)/(2 s \mu L_P)$.
  \item \emph{Long gaps.}
  $\bigl|[S_n^{(r)}]_{kj} - s^{-1} [(P_n^{(r)})^{-1}]_{kj}\bigr| \le s^{-2}\mu^{-2}$,
  so the leakage decays as $\Delta t_n^{-2}$.
  \item \emph{Exact form for $M = 2$.} With $g = a_1 a_2$ and
  $D := \alpha_1\alpha_2 - g^2 c^2 > 0$, \eqref{eq:S12} reads
  \begin{equation}\label{eq:S12-expanded}
    [S_n]_{12} = \frac{s\, g c}{1 + s(\alpha_1 + \alpha_2) + s^2 D},
  \end{equation}
  which is unimodal in $\Delta t_n$ with peak at
  $\Delta t_\star = D^{-1/4}$, peak value
  $|gc| / (\alpha_1 + \alpha_2 + 2\sqrt{D})$, and
  $[S_n]_{12} \sim g c / (\Delta t_n^2 D)$ as $\Delta t_n \to \infty$.
\end{enumerate}
Cross-modal transfer of the previous state is therefore largest at intermediate
gaps and vanishes at long ones.
\end{proposition}

\begin{proof}
Write $P := P_n^{(r)}$ and $S := S_n^{(r)}$. By Lemma~\ref{lem:psd},
applied channel-wise, $P$ is symmetric with spectrum in $[\mu, L_P]$, so $S$ is
symmetric with spectrum in $[(1 + s L_P)^{-1}, (1 + s\mu)^{-1}]$ and
$\|S\|_2 \le 1$.

(i) From $(I + sP)S = I$ we get $S = I - sP + s^2 P^2 S$. For $k \neq j$ the
identity contributes nothing and $[-sP]_{kj} = s\, a_k a_j c_{kj,r}$, while
$|[s^2 P^2 S]_{kj}| \le s^2 \|P\|_2^2 \|S\|_2 \le s^2 L_P^2$.

(ii) If $X$ is symmetric with spectrum in $[\lambda_-, \lambda_+]$, then for
$k \neq j$ and $m := (\lambda_- + \lambda_+)/2$,
$|X_{kj}| = |e_k^\top (X - mI) e_j| \le \|X - mI\|_2 = (\lambda_+ - \lambda_-)/2$.
Applying this to $S$ gives \eqref{eq:leak-bound}. Differentiating,
$\operatorname{sign}\beta'(s) = \operatorname{sign}\bigl(1 - s^2 \mu L_P\bigr)$,
which gives the monotonicity and the maximizer $s = (\mu L_P)^{-1/2}$;
substituting yields the stated maximum. The decay bound follows from
$(1 + s\mu)(1 + s L_P) \ge s^2 \mu L_P$.

(iii) From $(I + sP)S = I$ we also get $S = s^{-1}P^{-1} - s^{-1}P^{-1}S$, and
$\|P^{-1}S\|_2 \le \mu^{-1}(1 + s\mu)^{-1} \le s^{-1}\mu^{-2}$.

(iv) Expanding the denominator of \eqref{eq:S12} gives \eqref{eq:S12-expanded}.
Assumption~\ref{as:budget} gives $|c| \le (1-\epsilon)\min\{\alpha_1, \alpha_2\}$,
so $g^2 c^2 < \alpha_1 \alpha_2$ and $D > 0$. With $f(s) = s g c / Q(s)$ and
$Q(s) = 1 + s(\alpha_1 + \alpha_2) + s^2 D$, we have
$\operatorname{sign} f'(s) = \operatorname{sign}(gc)\cdot\operatorname{sign}\bigl(Q(s) - sQ'(s)\bigr)
= \operatorname{sign}(gc)\cdot\operatorname{sign}(1 - s^2 D)$, so $|f|$ peaks at
$s = D^{-1/2}$, where $Q = 2 + (\alpha_1 + \alpha_2)D^{-1/2}$, giving the stated
value. The asymptotic form follows because $Q(s) \sim s^2 D$.
\end{proof}

\begin{proposition}[Normal modes and frequency shift]
\label{prop:spectral-structure}
Fix a channel $r$ and an availability pattern $a$, and let
$P^{(r)}(a) = U^{(r)}\Lambda^{(r)}U^{(r)\top}$ with
$\Lambda^{(r)} = \diag(\nu_1^{(r)},\dots,\nu_M^{(r)})$. Then the unforced system
\eqref{eq:ode} restricted to channel $r$ is a superposition of $M$ decoupled
harmonic modes with angular frequencies $\sqrt{\nu_i^{(r)}}$ and mode shapes the
columns of $U^{(r)}$; the frequencies are pinned to the uncoupled ones by Weyl's
inequality, $|\nu_i^{(r)} - \alpha_{(i),r}| \le \gamma \le
(1-\epsilon)\alpha_{\max}$; and the channel transfer function is
$H^{(r)}(s) = (s^2 I_M + P^{(r)}(a))^{-1}$, whose cross-modal entry
$[H^{(r)}(s)]_{kj}$ vanishes for $k \neq j$ when modalities $k$ and $j$ lie in
different connected components of the gated graph $P^{(r)}(a)$, and is nonzero
for generic couplings when they lie in the same component. The entries
$|c_{kj,r}|$ therefore define $d$ weighted modality graphs whose connected
components are the sets of modalities that exchange information in that
frequency band.
\end{proposition}

\begin{proof}
(i) With $f = 0$, substituting $y = U^{(r)}\zeta$ into the channel-$r$
restriction of \eqref{eq:P} gives $\zeta'' = -\Lambda^{(r)}\zeta$, which is $M$
decoupled scalar oscillators with angular frequency $\sqrt{\nu_i^{(r)}}$, all
$\nu_i^{(r)} > 0$ by Lemma~\ref{lem:psd}. (ii)
$P^{(r)}(a) = \diag(\alpha_{\cdot,r}) - D_aC^{(r)}D_a$ is a symmetric
perturbation of a diagonal matrix by a symmetric matrix of spectral norm at most
$\gamma$, so Weyl's inequality \citep[Thm.~4.3.1]{horn2012matrix} gives the
stated bound. (iii) Laplace transforming $y'' = -P^{(r)}(a)y + \tilde f$ with
zero initial conditions gives $H^{(r)}(s) = (s^2I + P^{(r)}(a))^{-1}$. If $k$
and $j$ are disconnected in $P^{(r)}(a)$, then $s^2I + P^{(r)}(a)$ is block
diagonal and so is its inverse, giving $[H]_{kj} = 0$; conversely, if they are
connected by a path $k = i_0,\dots,i_q = j$, the Neumann series
$(s^2I+P)^{-1} = s^{-2}\sum_{q\ge0}(-s^{-2}P)^q$ contains the monomial
$P_{i_0i_1}\cdots P_{i_{q-1}i_q} \neq 0$ at order $q$, and for generic couplings
the entry is a nonzero rational function of $s$.
\end{proof}

\section{Parallel Evaluation, Complexity and Expressivity}
\label{app:scan}

Equation~\eqref{eq:recurrence} is a linear recurrence with time-varying
coefficients, so it is solvable by the associative scan with binary operator
\begin{equation}
  (T_1, v_1) \bullet (T_2, v_2) := \bigl(T_2 T_1,\; T_2 v_1 + v_2\bigr),
  \label{eq:scanop}
\end{equation}
which is associative
\citep{kogge1973parallel,blelloch1990prefix,martin2018parallel,smith2023s5}.
Applying its inclusive scan to
$[(\Top_1, \Top_1 F_1), \dots, (\Top_N, \Top_N F_N)]$ returns $x^1,\dots,x^N$ in
the second components in $\lceil \log_2 N \rceil$ sequential rounds.
Associativity is immediate; what is not immediate, and what the channel
factorization buys, is that the scan never leaves a cheap structural family.

\begin{definition}[Channel-block family]\label{def:family}
Let $\Fcal \subset \Rb^{2m \times 2m}$ be the set of matrices whose four
$m\times m$ blocks $T_{ab}$ satisfy
$\Pi T_{ab} \Pi^\top = \blkdiag(T_{ab}^{(1)},\dots,T_{ab}^{(d)})$ with
$T_{ab}^{(r)} \in \Rb^{M\times M}$.
\end{definition}

\begin{lemma}[Closure]\label{lem:closure}
$\Fcal$ is closed under addition and multiplication, contains $I_{2m}$, and
contains $\Top_n$ for every $n$ and every availability pattern. A product of two
elements of $\Fcal$ costs $8dM^3$ multiply-accumulate operations, against
$8m^3 = 8d^3M^3$ for dense $2m\times 2m$ blocks.
\end{lemma}

\begin{proof}
Let $T, T' \in \Fcal$. Conjugating by $\blkdiag(\Pi,\Pi)$ maps each of the four
$m\times m$ blocks of $T$ and $T'$ to block-diagonal matrices with $d$ blocks of
size $M\times M$. Block-diagonal matrices with a common partition are closed
under addition and multiplication, and each block of the product $TT'$ is a sum
of two products of such matrices, hence again block diagonal with the same
partition; therefore $TT' \in \Fcal$. Additivity and $I_{2m} \in \Fcal$ are
immediate. That $\Top_n \in \Fcal$ follows from \eqref{eq:recurrence}: $S_n$ and
$P_n$ are channel-block-diagonal by \eqref{eq:channelblock} and
Lemma~\ref{lem:psd}, and every block of $\Top_n$ is a scalar multiple of $S_n$
or of $P_nS_n$. For the cost, forming $TT'$ requires $8$ products of
$m \times m$ channel-block-diagonal matrices, each of which is $d$ independent
$M\times M$ products at $M^3$ multiply-accumulates, giving $8dM^3$; for dense
blocks the same $8$ products cost $m^3 = d^3M^3$ each.
\end{proof}

\begin{proposition}[Complexity and parameter count]\label{prop:cost}
One GCO layer requires $\Theta(N d M^3)$ arithmetic work and
$\lceil \log_2 N\rceil$ sequential depth, with $\Theta(N d M^2)$ memory for the
scan carries. An uncoupled oscillatory layer of the same total state width
$m = Md$ requires $\Theta(Nm)$ work, so the overhead of coupling is a factor
$\Theta(M^2)$, independent of $N$ and $d$, and the within-visit attention adds
$\Theta(N M^2 d)$ at constant depth. Of the $\Theta(M^2h^2 + M^2dh)$ parameters
per layer, only $\binom{M}{2}d$ belong to the coupling: for $M=4$ and $d=h=64$
that is $384$ out of $2.8\times10^5$, about $0.14\%$.
\end{proposition}

\begin{proof}
By \eqref{eq:recurrence} each $S_n^{(r)}$ is an $M\times M$ symmetric positive
definite inverse costing $\Theta(M^3)$ by Cholesky, and over $d$ channels and
$N$ visits this is $\Theta(NdM^3)$. Each $\Top_n$ lies in $\Fcal$ by
Lemma~\ref{lem:closure}, whose products cost $\Theta(dM^3)$; the associative
scan performs $\Theta(N)$ such operations at depth $\lceil\log_2 N\rceil$,
giving $\Theta(NdM^3)$ work and $\Theta(\log N)$ depth. Each scan carry stores
$4d$ blocks of size $M\times M$, hence $\Theta(NdM^2)$ memory. For the uncoupled
model, $C = 0$ makes every $S_n^{(r)}$ diagonal, so all products cost
$\Theta(dM) = \Theta(m)$ and the total is $\Theta(Nm)$. The attention block
\eqref{eq:attn} computes $M\times M$ scores at each of $N$ visits with
$d$-dimensional values, hence $\Theta(NM^2d)$ at constant depth.

For the parameter count, the free parameters of one GCO layer are
$\hat\alpha \in \Rb^{M \times d}$; $\hat c$ on the strict upper triangle in
$(k,j)$ for each channel, giving $\binom{M}{2}d$; the forcing matrices
$B_k \in \Rb^{d\times h}$ for $k \in [M]$, giving $Mdh$; the readout
$C_{\mathrm{ro}} \in \Rb^{Mh \times m}$ and skip
$D_{\mathrm{ro}} \in \Rb^{Mh \times Mh}$, giving $Mhm + M^2h^2$; and the two GLU
matrices in $\Rb^{Mh\times Mh}$, giving $2M^2h^2$. Since $m = Md$, we have
$Mhm = M^2dh$, so the total is
$Md + \binom{M}{2}d + M^2dh + Mdh + 3M^2h^2 = \Theta(M^2h^2 + M^2dh)$.
Substituting $M=4$ and $d=h=64$ gives
$256 + 384 + 65536 + 16384 + 196608 = 279{,}168 \approx 2.8\times10^5$, of which
the $\binom{4}{2}\cdot 64 = 384$ coupling parameters are $0.14\%$.
\end{proof}

Expressivity is not lost. Setting $\hat c \equiv 0$ in \eqref{eq:param} gives
$C = 0$ and reduces a CAMOS layer to an uncoupled oscillatory layer with
strictly positive diagonal state matrix, so CAMOS blocks inherit universality
within continuous causal operators. This
certifies only that Assumptions~\ref{as:sym} and \ref{as:budget} cost nothing in
expressivity, and it is weak here because a symmetric coupled system is
orthogonally equivalent to an uncoupled one when everything is observed. The
representational gain is therefore due to gating
(Theorem~\ref{thm:additivity}), since the gate acts in the modality basis and
does not commute with the diagonalizing rotation $U^{(r)}$ of
Proposition~\ref{prop:spectral-structure}.

\section{Dense and Low-Rank Couplings}
\label{app:dense}

If $C_{kj} \in \Rb^{d\times d}$ is dense rather than diagonal, $P$ does not
block-diagonalize under $\Pi$ and three things change.

\paragraph{Stability.} Assumption~\ref{as:budget} is replaced by a spectral
condition. If $C = C^\top$ and $\|C\|_2 \le (1-\epsilon)\lambda_{\min}(A)$, then
for every $a \in [0,1]^M$,
$\lambda_{\min}(P(a)) \ge \lambda_{\min}(A) - \|C\|_2 \ge
\epsilon\lambda_{\min}(A) > 0$ using $\|D_a\|_2 \le 1$, and
Lemma~\ref{lem:psd} and everything downstream go through with
$\mu := \epsilon\lambda_{\min}(A)$. Enforcing this requires spectral
normalization of $C$, hence a power iteration at every training step rather than
the closed form \eqref{eq:param}. It also interacts badly with the standard
$\alpha \sim \mathcal{U}([0,1])$ initialization, since $\lambda_{\min}(A)$ is
then near zero and the admissible coupling collapses.

\paragraph{Cost.} $S_n$ becomes a dense $m \times m$ inverse costing
$\Theta(d^3M^3)$, and the scan combine leaves $\Fcal$ and requires dense
$2m\times2m$ products at the same order. The overhead relative to the
channel-factored model is $\Theta(d^2)$, about a factor of $4000$ per layer at
$d = 64$.

\paragraph{Why low rank does not help.} The parameterization
$C_{kj} = Q_{kj}Q_{kj}^\top$ with $Q_{kj}\in\Rb^{d\times \rho}$ reduces the
parameter count from $d^2$ to $d\rho$ per block but not the scan cost, because
low-rank-plus-diagonal structure is not closed under multiplication: if
$T = D_1 + U_1V_1^\top$ and $T' = D_2 + U_2V_2^\top$ with
$\mathrm{rank}\,U_i = \rho$, then
$TT' = D_1D_2 + D_1U_2V_2^\top + U_1V_1^\top D_2 + U_1(V_1^\top U_2)V_2^\top$
has rank up to $3\rho$ in its non-diagonal part, so the carried operator's rank
triples at every combine and reaches $N^{\log_2 3}\rho$ after a full scan.
Woodbury inversion is available at each individual step but does not survive the
scan. This is worth stating because a low-rank PSD coupling is the natural first
choice, and it is where intuition from static fusion layers fails in a
sequential model.

\paragraph{Intermediate option.} Partitioning the $d$ channels into $d/\rho$
groups with dense coupling inside each group interpolates between the extremes:
$\Fcal$ becomes block-diagonal matrices with $d/\rho$ blocks of size $\rho M$,
the cost becomes $\Theta(N d \rho^2 M^3)$, and $\rho = 1$ recovers the
channel-factored model.

\section{Algorithms}
\label{app:algo}

\begin{algorithm}[h]
\caption{One gated coupled oscillator (GCO) layer, forward pass}
\label{alg:gco}
\begin{algorithmic}[1]
\Require projected inputs $\{\tilde u_k^n\}$, availability $\{a_k^n\}$, gaps
  $\{\Delta t_n\}$, initial state $x^0$, parameters
  $\hat\alpha, \hat c, \{B_k\}$, budget $\epsilon$, floor $\delta$
\Ensure states $\{x^n = (z^n, y^n)\}_{n=1}^N$
\State $\alpha_{k,r} \gets \spls(\hat\alpha_{k,r})$ for all $k, r$
  \Comment{$\Theta(Md)$}
\State $\sigma_{k,r} \gets \sum_{l\neq k}|\hat c_{kl,r}|$;\quad
  $c_{kj,r} \gets \dfrac{(1-\epsilon)\min\{\alpha_{k,r},\alpha_{j,r}\}}
                        {\max\{\sigma_{k,r},\sigma_{j,r},\delta\}}\,\hat c_{kj,r}$
  \Comment{Assumptions~\ref{as:sym}, \ref{as:budget} hold by construction}
\For{$n = 1,\dots,N$ \textbf{in parallel}}
  \For{$r = 1,\dots,d$ \textbf{in parallel}}
    \State $[P_n^{(r)}]_{kk} \gets \alpha_{k,r}$;\quad
           $[P_n^{(r)}]_{kj} \gets -\,a_k^n a_j^n\, c_{kj,r}$ for $k\neq j$
           \Comment{symmetric product gate}
    \State $S_n^{(r)} \gets \textsc{CholSolve}\bigl(I_M + \Delta t_n^2 P_n^{(r)},\, I_M\bigr)$
           \Comment{SPD by Lemma~\ref{lem:psd}; $\Theta(M^3)$}
  \EndFor
  \State assemble $\Top_n$ from $\{S_n^{(r)}\}, \{P_n^{(r)}\}$ by \eqref{eq:recurrence}
  \State $f^n \gets \sum_k a_k^n\,\iota_k B_k \tilde u_k^n$;\quad
         $v_n \gets \Top_n (\Delta t_n f^n,\, 0)^\top$
\EndFor
\State $\{x^n\} \gets \textsc{AssocScan}_{\bullet}\bigl([(\Top_1, v_1 + \Top_1 x^0),
  (\Top_2, v_2), \dots, (\Top_N, v_N)]\bigr)$
  \Comment{$\lceil\log_2 N\rceil$ depth, $\Theta(NdM^3)$ work}
\State \Return $\{x^n\}$
\end{algorithmic}
\end{algorithm}

\begin{algorithm}[h]
\caption{CAMOS forward pass for one subject}
\label{alg:camos}
\begin{algorithmic}[1]
\Require $\{u_k^n\}$, $\{a_k^n\}$, $\{t_n\}$, static $s$
\Ensure $\{\hat p^n\}$, $\hat q$, $\{\hat u_k^{n+h}\}$
\State $\Delta t_n \gets t_n - t_{n-1}$ for all $n$
\State $\tilde u_k^{(0),n} \gets E_k u_k^n + e_k$ if $a_k^n = 1$, else $0$
  \Comment{\eqref{eq:proj}, parallel in $n$}
\State $x^0 \gets W_2^s\,\mathrm{GELU}(W_1^s s + b_1^s) + b_2^s$
  \Comment{\eqref{eq:static}}
\For{$\ell = 1,\dots,L$}
  \State $\{x^{(\ell),n}\} \gets$ Algorithm~\ref{alg:gco} with inputs
    $\{\tilde u^{(\ell-1),n}\}$ and layer-$\ell$ parameters
  \State $w^{(\ell),n} \gets C^{(\ell)}_{\mathrm{ro}} y^{(\ell),n}
    + D^{(\ell)}_{\mathrm{ro}} \tilde u^{(\ell-1),n}$
  \State $\tilde u^{(\ell),n} \gets \tilde u^{(\ell-1),n}
    + \mathrm{Drop}(\mathrm{GLU}(\mathrm{GELU}(w^{(\ell),n})))$
\EndFor
\State $h_k^n \gets \iota_k^\top y^{(L),n}$;\quad
  $\bar h_k^n \gets h_k^n + \mathrm{Attn}(h^n)_k$
  \Comment{\eqref{eq:attn}, parallel in $n$}
\State $\hat p^n, \hat q, \hat u_k^{n+h} \gets$ heads
  \eqref{eq:head-st}, \eqref{eq:head-cv}, \eqref{eq:head-fc}
\end{algorithmic}
\end{algorithm}

\section{Implementation Details}
\label{app:impl}

All models are implemented in PyTorch 2.3 \citep{paszke2019pytorch} and trained
on a single NVIDIA A100 40GB GPU, with the associative scan implemented as a
Blelloch up-sweep and down-sweep over batched $2M\times2M$ channel blocks and
the Cholesky solves of \eqref{eq:recurrence} batched over $(n,r)$ in float32.
Table~\ref{tab:hparams} lists the complete configuration; values marked
``grid'' are selected on the validation split, and the identical grid and trial
budget are applied to every baseline so that no model receives a larger search.
The validation selection criterion is the unweighted mean of staging macro-F1,
landmark prediction AUROC and one minus the normalized forecasting MAE. Staging is scored
by accuracy and macro-averaged F1, precision, recall and specificity;
landmark prediction by AUROC and AUPRC; forecasting by per-modality MAE and RMSE in the
normalized feature space, averaged over observed targets only.

The frequency shift $\alpha_0 > 0$ in Table~\ref{tab:hparams} matters: the
standard oscillatory initialization $\alpha \sim \mathcal{U}([0,1])$ places mass
near zero, and the admissible coupling in \eqref{eq:param} is proportional to
$\min\{\alpha_{k,r},\alpha_{j,r}\}$, so a large fraction of channels would begin
with no coupling budget and never recover.

\begin{table}[h]
\caption{Complete CAMOS configuration.}
\label{tab:hparams}
\centering
\small
\begin{tabular}{@{}llp{0.44\textwidth}@{}}
\toprule
Group & Setting & Value \\
\midrule
\multirow{6}{*}{Architecture}
 & GCO blocks $L$              & grid $\{2, 4, 6\}$ \\
 & oscillators per modality $d$& grid $\{32, 64, 128\}$ \\
 & projection width $h$        & grid $\{32, 64, 128\}$ \\
 & attention heads $H$         & $4$ \\
 & dropout                     & $0.1$ \\
 & static encoder              & $2$-layer MLP, width $2m$, GELU \\
\midrule
\multirow{4}{*}{Coupling}
 & budget $\epsilon$           & $0.2$ \\
 & numerical floor $\delta$    & $10^{-6}$ \\
 & gate                        & hard, $g_{kj}^n = a_k^n a_j^n$ \\
 & structure                   & channel-factored, symmetric \\
\midrule
\multirow{4}{*}{Initialization}
 & $\alpha_{k,r}$              & $\mathcal{U}([\alpha_0, \alpha_0+1])$, $\alpha_0 = 0.1$ \\
 & $\hat c_{kj,r}$             & $\mathcal{N}(0, 1/M)$, symmetrized \\
 & attention bias $\beta_0$    & $2.0$, learnable via $\spls$ \\
 & remaining weights           & PyTorch defaults \\
\midrule
\multirow{7}{*}{Optimization}
 & optimizer                   & AdamW, $\beta = (0.9, 0.999)$, $\varepsilon=10^{-8}$ \\
 & weight decay                & $10^{-2}$, excluding $\hat\alpha$, $\hat c$, biases \\
 & peak learning rate          & grid $\{10^{-4}, 3\cdot10^{-4}, 10^{-3}\}$ \\
 & schedule                    & cosine, $5\%$ linear warmup \\
 & batch size                  & $64$ subjects \\
 & gradient clipping           & global norm $1.0$ \\
 & epochs / early stopping     & max $200$, patience $20$ \\
\midrule
\multirow{4}{*}{Objective}
 & $(\lambda_{\mathrm{st}}, \lambda_{\mathrm{cv}}, \lambda_{\mathrm{fc}})$
                               & $(1,\ 0.5,\ 1)$ \\
 & $\lambda_{\mathrm{reg}}$    & $10^{-3}$ ($\ell_1$ on $\hat c$) \\
 & forecasting horizons        & $H_{\max} = 3$, discount $\gamma_h = 2^{-(h-1)}$ \\
 & class weighting             & inverse frequency on the training split \\
\midrule
\multirow{4}{*}{Protocol}
 & split                       & subject-level $70/15/15$, stratified \\
 & seeds                       & $5$ \\
 & selection criterion         & mean of staging macro-F1, landmark prediction AUROC,
                                 $1 - $ normalized forecasting MAE \\
 & precision                   & float32 throughout, including Cholesky \\
\bottomrule
\end{tabular}
\end{table}

\paragraph{Numerical notes.} The Cholesky factorizations in
\eqref{eq:recurrence} are the only place where conditioning could be a concern.
By Lemma~\ref{lem:psd}, $I_M + \Delta t_n^2 P_n^{(r)}$ has condition number at
most $(1 + \Delta t_n^2 L_P)/(1 + \Delta t_n^2\mu)$, bounded above by
$L_P/\mu = (2-\epsilon)\,\mathrm{cond}(A)/\epsilon$ independently of
$\Delta t_n$; with $\epsilon = 0.2$ and the initialization of
Table~\ref{tab:hparams} this is below $200$, comfortably inside float32. This is
a practical dividend of the budget: an unconstrained coupling would allow
$I + \Delta t^2 P$ to become near-singular precisely at long gaps, where the
model is most needed.

\paragraph{Baseline adaptation.} LinOSS-IM and LinOSS-IMEX receive the
concatenation $[u_1^n; \dots; u_M^n] \in \Rb^{p}$ with unobserved entries set to
zero, the availability vector $a^n$ appended as $M$ channels, and $\Delta t_n$
appended as one channel and used as the per-step discretization interval.
MedFuse and DrFuse are given per-modality gated recurrent encoders over the
visit sequence, so that they see the same temporal information as the SSMs
rather than a single pooled vector.

\section{Dataset Construction}
\label{app:data}

\paragraph{ADNI extraction.} From the ADNIMERGE table we retain the columns
listed in Section~\ref{sec:datasets} together with the visit code, examination
date, baseline diagnosis and current diagnosis. Visits are ordered by
examination date and $t_n$ is expressed in years from the subject's baseline
visit. A subject is included if it has at least three visits with a recorded
diagnosis and, at each retained visit, at least one modality fully observed.
Volumetric MRI measures are divided by intracranial volume before normalization,
which removes head-size variation that would otherwise dominate the leading
principal direction. All continuous features are then robust $z$-scored,
$u \mapsto (u - \mathrm{med})/\mathrm{IQR}$, with statistics computed on the
training split only and applied unchanged to validation and test; the robust
form is preferred because several ADNI biomarkers, CSF amyloid, are
censored at assay limits and have heavy tails. Categorical static covariates are
one-hot encoded, and \textit{APOE4} is kept as its integer allele count in
$\{0,1,2\}$, which encodes the known dose response.

\begin{table}[h]
\caption{Cohort statistics. $M$: number of modalities; $p=\sum_k p_k$: total raw
feature dimension; missing rate: percentage of modality-visit pairs unobserved.}
\label{tab:datasets}
\centering\small
\setlength{\tabcolsep}{4pt}
\begin{tabular}{llcccccc c}
\toprule
Dataset & Role & $M$ & $p$ & Subjects & Visits & Mean visits/subj. & Max visits/subj. & Missing (\%) \\
\midrule
ADNI    & train / val / test & 4 & 19 & 2749 & 7432 & 2.7 & 15 & 32.5 \\
OASIS-3 & zero-shot          & 3 & 19 & 1377 & 8522  & 6.2 & 20 & 28 \\
\bottomrule
\end{tabular}
\end{table}

\paragraph{Availability convention.} $a_k^n = 1$ if and only if every feature of
modality $k$ is recorded at visit $n$. Per-feature masking was rejected because
acquisition is organized by procedure: a lumbar puncture yields the whole CSF
panel or none of it, and a PET session yields the whole tracer panel or none of
it. Modelling availability at the procedure level is therefore both more
faithful and what makes the modality-level gate of Section~\ref{sec:gating}
meaningful.

\paragraph{Label construction.} Staging labels are the recorded diagnosis at
each visit, mapped to $\{\mathrm{CN}, \mathrm{MCI}, \mathrm{AD}\}$; visits
without a recorded diagnosis contribute to the forecasting loss but not to
$\mathcal{L}_{\mathrm{st}}$. For landmark prediction, the index visit is the first visit
at which a subject is diagnosed MCI; the label is $1$ if AD is recorded at any
visit within three years of the index visit, and $0$ if the subject has at least
three years of subsequent follow-up without an AD diagnosis. Subjects censored
before three years without converting are excluded from the landmark prediction task
only, so no conversion label is ever imputed. Reversions from MCI to CN are
retained in the staging task with their recorded label and do not affect the
conversion label unless a later AD diagnosis occurs within the window.

\paragraph{OASIS-3 mapping.} For OASIS-3 \citep{lamontagne2019oasis3} we use the
same modalities and, where a feature has no exact ADNI counterpart, the nearest
available measure of the same construct, refitting only the normalization
statistics; the model is otherwise transferred without retraining, so this is a
zero-shot evaluation.

\section{Extended Results}
\label{app:extended}

In this section, we report all the baselines on both cohorts, across the four
families: uncoupled oscillatory SSMs (LinOSS-IM,
LinOSS-IMEX, D-LinOSS), general and coupled SSMs (S5, Mamba, Coupled Mamba),
irregular-time and missing-value models (GRU-D, mTAN, Raindrop, ODE-RNN), and
clinical fusion models (MedFuse, DrFuse). All entries are means over five seeds
on the subject-level splits of Appendix~\ref{app:data}.

\subsection{Staging with all the baselines}
\label{app:staging-ext}

Table~\ref{tab:staging_extended} extends the staging comparison of
Section~\ref{sec:res-staging}. On ADNI the strongest baselines come from three
different families, DrFuse ($0.7322$ accuracy), LinOSS-IMEX ($0.7206$) and
Mamba ($0.7187$), so no single design axis accounts for baseline performance,
and CAMOS exceeds all twelve on every metric. On OASIS-3 the picture changes:
accuracy is compressed between $0.6263$ and $0.7525$ because one class
dominates, while macro F1 separates the models from $0.2805$ to $0.5557$. Every
baseline falls below macro F1 $0.44$ and macro recall $0.44$, with DrFuse at
exactly $1/3$ recall and $2/3$ specificity, the signature of predicting a single
class. CAMOS is the only model that retains discrimination among the three
stages under transfer.

\begin{table}[t]
\caption{Same-visit CN/MCI/AD staging on ADNI (in-distribution) and OASIS-3
(zero-shot).}
\label{tab:staging_extended}
\centering
\small
\setlength{\tabcolsep}{3pt}
\begin{tabular}{lccccc}
\toprule
Models & Accuracy \up & Macro F1 \up & Macro Prec.\ \up & Macro Recall \up & Macro Spec.\ \up \\
\midrule
\multicolumn{6}{l}{\textbf{ADNI}} \\
\cmidrule(r){1-1}
LinOSS-IM & $0.6593_{\pm 0.0041}$ & $0.6815_{\pm 0.0039}$ & $0.6804_{\pm 0.0045}$ & $0.6864_{\pm 0.0037}$ & $0.8147_{\pm 0.0021}$ \\
LinOSS-IMEX & $0.7206_{\pm 0.0058}$ & $0.7234_{\pm 0.0061}$ & $0.7394_{\pm 0.0055}$ & $0.7424_{\pm 0.0059}$ & $0.8513_{\pm 0.0030}$ \\
D-LinOSS & $0.6820_{\pm 0.0324}$ & $0.6882_{\pm 0.0271}$ & $0.6912_{\pm 0.0294}$ & $0.6975_{\pm 0.0281}$ & $0.8392_{\pm 0.0152}$ \\
S5 & $0.6562_{\pm 0.0226}$ & $0.6696_{\pm 0.0196}$ & $0.6731_{\pm 0.0207}$ & $0.6804_{\pm 0.0196}$ & $0.8215_{\pm 0.0118}$ \\
Mamba & $0.7187_{\pm 0.0203}$ & $0.7304_{\pm 0.0163}$ & $0.7386_{\pm 0.0174}$ & $0.7412_{\pm 0.0169}$ & $0.8561_{\pm 0.0097}$ \\
Coupled Mamba & $0.7002_{\pm 0.0282}$ & $0.7052_{\pm 0.0209}$ & $0.7127_{\pm 0.0233}$ & $0.7169_{\pm 0.0214}$ & $0.8447_{\pm 0.0126}$ \\
GRU-D & $0.6192_{\pm 0.0356}$ & $0.6251_{\pm 0.0391}$ & $0.6318_{\pm 0.0372}$ & $0.6402_{\pm 0.0365}$ & $0.8043_{\pm 0.0198}$ \\
mTAN & $0.6931_{\pm 0.0182}$ & $0.7037_{\pm 0.0165}$ & $0.7094_{\pm 0.0171}$ & $0.7158_{\pm 0.0163}$ & $0.8489_{\pm 0.0092}$ \\
Raindrop & $0.6776_{\pm 0.0174}$ & $0.6902_{\pm 0.0172}$ & $0.6953_{\pm 0.0182}$ & $0.7021_{\pm 0.0176}$ & $0.8374_{\pm 0.0101}$ \\
ODE-RNN & $0.6619_{\pm 0.0194}$ & $0.6682_{\pm 0.0172}$ & $0.6729_{\pm 0.0184}$ & $0.6817_{\pm 0.0179}$ & $0.8268_{\pm 0.0105}$ \\
MedFuse & $0.6648_{\pm 0.0072}$ & $0.6697_{\pm 0.0076}$ & $0.6564_{\pm 0.0069}$ & $0.7139_{\pm 0.0070}$ & $0.8284_{\pm 0.0034}$ \\
DrFuse & $0.7322_{\pm 0.0049}$ & $0.7354_{\pm 0.0046}$ & $0.7488_{\pm 0.0044}$ & $0.7683_{\pm 0.0042}$ & $0.8599_{\pm 0.0024}$ \\
\textbf{CAMOS} & $\best{0.7824}_{\pm \best{0.0016}}$ & $\best{0.7790}_{\pm \best{0.0018}}$ & $\best{0.7970}_{\pm \best{0.0022}}$ & $\best{0.7786}_{\pm \best{0.0025}}$ & $\best{0.8836}_{\pm \best{0.0013}}$ \\
\midrule
\multicolumn{6}{l}{\textbf{OASIS-3}} \\
\cmidrule(r){1-1}
LinOSS-IM & $0.7271_{\pm 0.0017}$ & $0.3253_{\pm 0.0053}$ & $0.5334_{\pm 0.0082}$ & $0.3541_{\pm 0.0025}$ & $0.6773_{\pm 0.0014}$ \\
LinOSS-IMEX & $0.7231_{\pm 0.0022}$ & $0.3027_{\pm 0.0064}$ & $0.4222_{\pm 0.0094}$ & $0.3414_{\pm 0.0028}$ & $0.6756_{\pm 0.0016}$ \\
D-LinOSS & $0.7022_{\pm 0.0343}$ & $0.3754_{\pm 0.0376}$ & $0.4127_{\pm 0.0412}$ & $0.3896_{\pm 0.0284}$ & $0.6948_{\pm 0.0163}$ \\
S5 & $0.7525_{\pm 0.0290}$ & $0.4047_{\pm 0.0396}$ & $0.4518_{\pm 0.0437}$ & $0.4102_{\pm 0.0315}$ & $0.7051_{\pm 0.0172}$ \\
Mamba & $0.7345_{\pm 0.0262}$ & $0.4312_{\pm 0.0492}$ & $0.4863_{\pm 0.0521}$ & $0.4389_{\pm 0.0406}$ & $0.7194_{\pm 0.0208}$ \\
Coupled Mamba & $0.7237_{\pm 0.0636}$ & $0.3594_{\pm 0.0667}$ & $0.3924_{\pm 0.0703}$ & $0.3781_{\pm 0.0452}$ & $0.6890_{\pm 0.0231}$ \\
GRU-D & $0.7259_{\pm 0.0494}$ & $0.3731_{\pm 0.0363}$ & $0.4035_{\pm 0.0398}$ & $0.3852_{\pm 0.0317}$ & $0.6926_{\pm 0.0159}$ \\
mTAN & $0.7371_{\pm 0.0376}$ & $0.2931_{\pm 0.0110}$ & $0.2846_{\pm 0.0137}$ & $0.3409_{\pm 0.0068}$ & $0.6705_{\pm 0.0034}$ \\
Raindrop & $0.6263_{\pm 0.0377}$ & $0.3411_{\pm 0.0298}$ & $0.3627_{\pm 0.0312}$ & $0.3597_{\pm 0.0243}$ & $0.6798_{\pm 0.0121}$ \\
ODE-RNN & $0.7154_{\pm 0.0403}$ & $0.3330_{\pm 0.0124}$ & $0.3468_{\pm 0.0154}$ & $0.3502_{\pm 0.0097}$ & $0.6751_{\pm 0.0048}$ \\
MedFuse & $0.6682_{\pm 0.0032}$ & $0.3176_{\pm 0.0059}$ & $0.2871_{\pm 0.0091}$ & $0.3558_{\pm 0.0026}$ & $0.6844_{\pm 0.0019}$ \\
DrFuse & $0.7263_{\pm 0.0023}$ & $0.2805_{\pm 0.0081}$ & $0.2421_{\pm 0.0086}$ & $0.3333_{\pm 0.0000}$ & $0.6667_{\pm 0.0000}$ \\
\textbf{CAMOS} & $\best{0.7600}_{\pm \best{0.0014}}$ & $\best{0.5557}_{\pm \best{0.0028}}$ & $\best{0.6280}_{\pm \best{0.0031}}$ & $\best{0.5261}_{\pm \best{0.0030}}$ & $\best{0.7682}_{\pm \best{0.0015}}$ \\
\bottomrule
\end{tabular}
\end{table}

\subsection{Landmark prediction with all the baselines}
\label{app:conversion-ext}

Table~\ref{tab:conversion_extended} extends Table~\ref{tab:conversion}, and
Figure~\ref{fig:roc_pr} shows the corresponding ROC and precision-recall curves.
The two metrics behave differently in each cohort. On ADNI, AUROC spans
$0.8795$ to $0.9611$ while AUPRC spans $0.6187$ to $0.8687$, and the ordering is
not preserved: mTAN is the strongest baseline on AUPRC ($0.7750$) but only
third on AUROC. On OASIS-3 the gap widens, with baseline AUPRC between $0.0801$
and $0.5982$ at AUROC between $0.7896$ and $0.9560$, indicating few positive
cases in the transfer cohort. Baseline standard deviations there reach
$\pm0.188$ on AUPRC, against $\pm0.009$ for CAMOS, so the transfer margin should
be read alongside that variance.

\begin{figure}[t]
\centering
\begin{subfigure}[t]{0.5\textwidth}
    \centering
    \includegraphics[width=\linewidth]{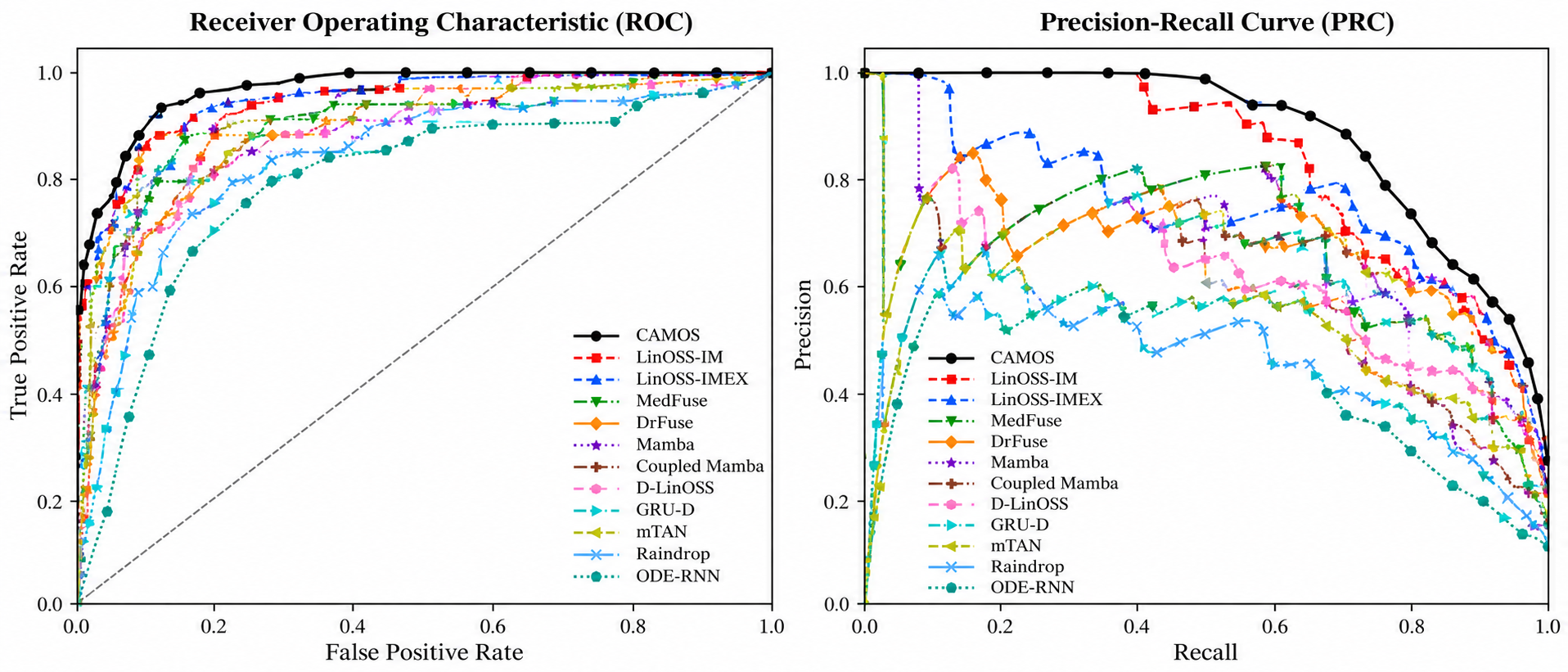}
\end{subfigure}
\hfill
\begin{subfigure}[t]{0.49\textwidth}
    \centering
    \includegraphics[width=\linewidth]{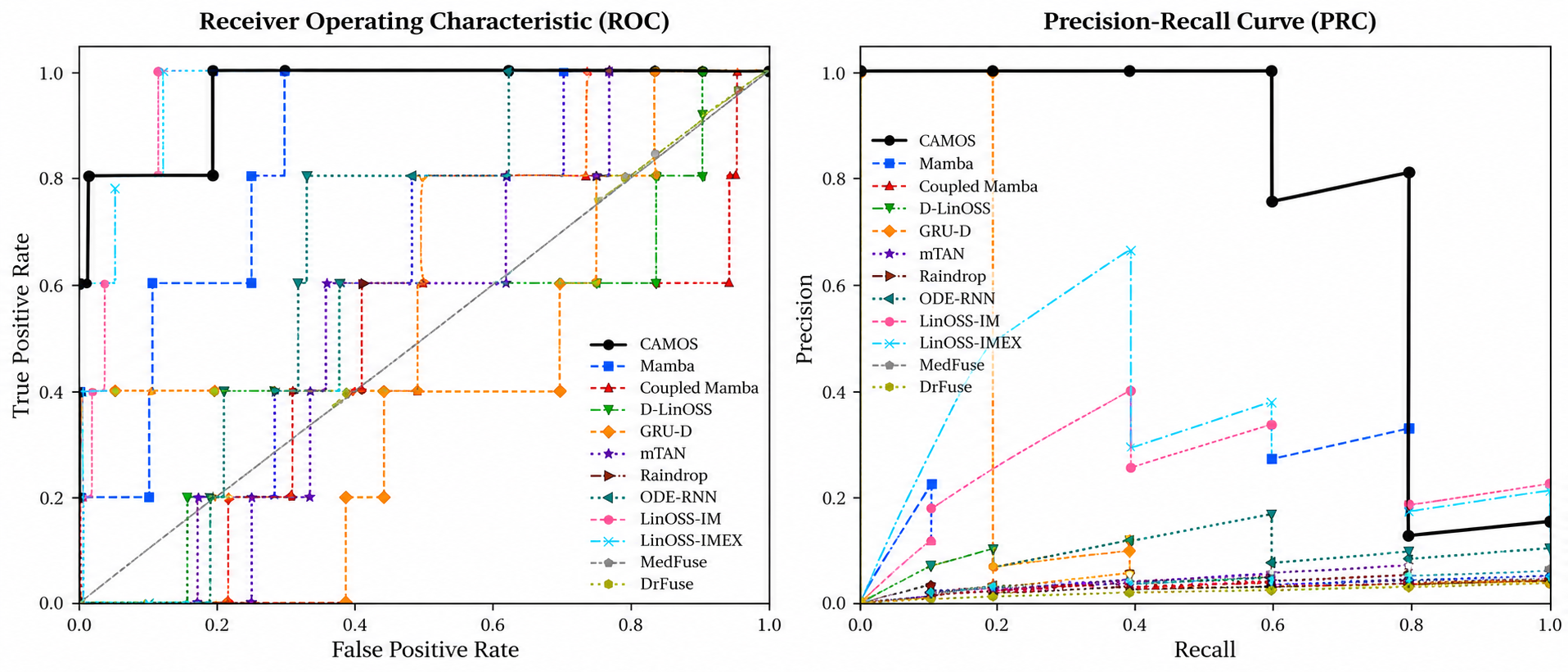}
\end{subfigure}
\caption{ROC and PR curves for 3-year MCI-to-AD landmark prediction on ADNI (left) and OASIS-3 (right).}
\label{fig:roc_pr}
\end{figure}

\begin{table}[t]
\caption{3-year MCI-to-AD conversion on ADNI (in-distribution) and OASIS-3
(zero-shot).}
\label{tab:conversion_extended}
\centering
\small
\setlength{\tabcolsep}{5pt}
\begin{tabular}{lcccc}
\toprule
& \multicolumn{2}{c}{\textbf{ADNI}} & \multicolumn{2}{c}{\textbf{OASIS-3}} \\
\cmidrule(lr){2-3}\cmidrule(lr){4-5}
Models & AUROC \up & AUPRC \up & AUROC \up & AUPRC \up \\
\midrule
LinOSS-IM     & $0.9223_{\pm 0.0031}$ & $0.7102_{\pm 0.0087}$ & $0.9533_{\pm 0.0024}$ & $0.3588_{\pm 0.0142}$ \\
LinOSS-IMEX   & $0.9307_{\pm 0.0027}$ & $0.7389_{\pm 0.0074}$ & $0.9560_{\pm 0.0029}$ & $0.4185_{\pm 0.0163}$ \\
D-LinOSS      & $0.9142_{\pm 0.0187}$ & $0.7022_{\pm 0.0721}$ & $0.8236_{\pm 0.0412}$ & $0.1134_{\pm 0.0481}$ \\
S5            & $0.9087_{\pm 0.0163}$ & $0.6925_{\pm 0.06201}$ & $0.9412_{\pm 0.0265}$ & $0.5982_{\pm 0.1889}$ \\
Mamba         & $0.8795_{\pm 0.0204}$ & $0.6187_{\pm 0.0686}$ & $0.8418_{\pm 0.0537}$ & $0.1260_{\pm 0.1104}$ \\
Coupled Mamba & $0.8913_{\pm 0.0142}$ & $0.6463_{\pm 0.0423}$ & $0.8307_{\pm 0.0581}$ & $0.1177_{\pm 0.1189}$ \\
GRU-D         & $0.9128_{\pm 0.0196}$ & $0.7014_{\pm 0.0678}$ & $0.8654_{\pm 0.0476}$ & $0.1825_{\pm 0.1054}$ \\
mTAN          & $0.9346_{\pm 0.0172}$ & $0.7750_{\pm 0.0667}$ & $0.8372_{\pm 0.0503}$ & $0.1224_{\pm 0.1031}$ \\
Raindrop      & $0.9061_{\pm 0.0158}$ & $0.6909_{\pm 0.0584}$ & $0.7896_{\pm 0.0348}$ & $0.0801_{\pm 0.0322}$ \\
ODE-RNN       & $0.9215_{\pm 0.0211}$ & $0.7250_{\pm 0.0752}$ & $0.8519_{\pm 0.0392}$ & $0.1308_{\pm 0.0676}$ \\
MedFuse       & $0.9134_{\pm 0.0045}$ & $0.6902_{\pm 0.0102}$ & $0.8467_{\pm 0.0058}$ & $0.2594_{\pm 0.0121}$ \\
DrFuse        & $0.9386_{\pm 0.0022}$ & $0.7689_{\pm 0.0069}$ & $0.8787_{\pm 0.0047}$ & $0.3183_{\pm 0.0135}$ \\
\textbf{CAMOS} & $\best{0.9611}_{\pm \best{0.0014}}$ & $\best{0.8687}_{\pm \best{0.0041}}$ & $\best{0.9600}_{\pm \best{0.0018}}$ & $\best{0.7894}_{\pm \best{0.0093}}$ \\
\bottomrule
\end{tabular}
\end{table}

\subsection{Longitudinal forecasting with all the baselines}
\label{app:forecast-ext}

Tables~\ref{tab:forecast_extended} and \ref{tab:oasis_forecasting} extend
Figure~\ref{fig:forecast} to all baselines. On ADNI, the baselines are tightly
clustered on the two sparse modalities, within $0.112$ MAE on CSF and, aside
from S5, Mamba and Coupled Mamba, within $0.091$ on genetics, while CAMOS
separates by $0.115$ and $0.106$ from the best of them. On OASIS-3 the
clustering is tighter still: eleven of twelve baselines lie within $0.041$ MAE
of each other on imaging and within $0.039$ on clinical measures, which suggests
they largely reproduce the last observed value. CAMOS departs from that cluster
on clinical measures ($0.3756$ against a baseline minimum of $0.6569$) and
genetics ($0.2847$ against $0.7711$), and only marginally on imaging ($0.7333$
against $0.7705$).

\begin{table}[t]
\caption{Next-visit ($h=1$) forecasting on ADNI per modality.}
\label{tab:forecast_extended}
\centering
\small
\setlength{\tabcolsep}{3.6pt}
\begin{tabular}{l cc cc cc cc}
\toprule
& \multicolumn{2}{c}{Cognition \dn} & \multicolumn{2}{c}{CSF \dn}
& \multicolumn{2}{c}{Genetics \dn} & \multicolumn{2}{c}{Imaging \dn} \\
\cmidrule(lr){2-3}\cmidrule(lr){4-5}\cmidrule(lr){6-7}\cmidrule(lr){8-9}
Model & MAE & RMSE & MAE & RMSE & MAE & RMSE & MAE & RMSE \\
\midrule
LinOSS-IM     & 0.4471 & 0.7704 & 0.7715 & 1.0832 & 0.8270 & 0.9662 & 0.4468 & 0.8457 \\
LinOSS-IMEX   & 0.4343 & 0.7526 & 0.7676 & 1.0795 & 0.8183 & 0.9607 & 0.5567 & 0.8053 \\
D-LinOSS      & 0.5579 & 0.8736 & 0.7651 & 1.0724 & 0.7802 & 0.9731 & 0.6017 & 0.8612 \\
S5            & 0.4973 & 0.8012 & 0.7591 & 1.0689 & 0.5147 & 0.9712 & 0.5945 & 0.8547 \\
Mamba         & 0.4932 & 0.7894 & 0.7056 & 1.0180 & 0.5964 & 0.9480 & 0.5646 & 0.8273 \\
Coupled Mamba & 0.4976 & 0.7948 & 0.7139 & 1.0518 & 0.6166 & 0.9688 & 0.5764 & 0.8394 \\
GRU-D         & 0.5695 & 0.8891 & 0.7349 & 1.0612 & 0.8011 & 0.9804 & 0.5867 & 0.8481 \\
mTAN          & 0.5728 & 0.8824 & 0.7517 & 1.0703 & 0.8000 & 0.9779 & 0.5668 & 0.8325 \\
Raindrop      & 0.6561 & 0.9967 & 0.8179 & 1.1386 & 0.7823 & 0.9856 & 0.6781 & 0.9412 \\
ODE-RNN       & 0.5578 & 0.8649 & 0.7471 & 1.0655 & 0.7360 & 0.9694 & 0.5771 & 0.8436 \\
MedFuse       & 0.4199 & 0.7136 & 0.7873 & 1.1452 & 0.7644 & 0.9743 & 0.4538 & 0.7828 \\
DrFuse        & 0.4176 & 0.6586 & 0.7895 & 1.0443 & 0.7535 & 0.9610 & 0.5562 & 0.7832 \\
\textbf{CAMOS} & \best{0.3588} & \best{0.6124} & \best{0.5906} & \best{0.8537} & \best{0.4088} & \best{0.5930} & \best{0.4146} & \best{0.6902} \\
\bottomrule
\end{tabular}
\end{table}

\begin{table}[hbt!]
\centering
\small
\caption{Next-visit ($h=1$) forecasting on OASIS-3 per modality.}
\label{tab:oasis_forecasting}
\begin{tabular}{l cc cc cc}
\toprule
& \multicolumn{2}{c}{Imaging $\downarrow$} & \multicolumn{2}{c}{Cognition $\downarrow$} & \multicolumn{2}{c}{Genetics $\downarrow$} \\
\cmidrule(lr){2-3} \cmidrule(lr){4-5} \cmidrule(lr){6-7}
Model & MAE & RMSE & MAE & RMSE & MAE & RMSE \\
\midrule
LinOSS-IM     & 0.7748 & 0.9968 & 0.6719 & 0.9356 & 0.7993 & 1.0398 \\
LinOSS-IMEX   & 0.7745 & 0.9966 & 0.6712 & 0.9352 & 0.8011 & 1.0408 \\
D-LinOSS      & 0.7721 & 0.9937 & 0.6789 & 0.9387 & 0.7904 & 1.0341 \\
S5            & 0.8116 & 1.0342 & 0.6916 & 0.9512 & 0.7840 & 1.0296 \\
Mamba         & 0.7866 & 1.0081 & 0.6867 & 0.9458 & 0.7711 & 1.0213 \\
Coupled Mamba & 0.7874 & 1.0094 & 0.6959 & 0.9547 & 0.8026 & 1.0419 \\
GRU-D         & 0.7846 & 1.0057 & 0.6854 & 0.9441 & 0.8054 & 1.0437 \\
mTAN          & 0.7705 & 0.9921 & 0.6749 & 0.9369 & 0.8048 & 1.0429 \\
Raindrop      & 0.8968 & 1.1273 & 0.8194 & 1.0836 & 0.9028 & 1.1542 \\
ODE-RNN       & 0.7902 & 1.0126 & 0.6775 & 0.9372 & 0.7974 & 1.0385 \\
MedFuse       & 0.7721 & 0.9925 & 0.6765 & 0.9194 & 0.8035 & 1.0377 \\
DrFuse        & 0.7729 & 0.9929 & 0.6569 & 0.9161 & 0.8158 & 1.0446 \\
\textbf{CAMOS} & \best{0.7333} & \best{0.9464} & \best{0.3756} & \best{0.6401} & \best{0.2847} & \best{0.4927} \\
\bottomrule
\end{tabular}
\end{table}

\subsection{Ablation studies on OASIS-3}
\label{app:ablation-oasis}

Table~\ref{tab:ablation_oasis} repeats the ablations of
Section~\ref{sec:ablation} under zero-shot transfer, and the ordering matches
ADNI. The IMEX discretization is the most damaging change, lowering macro F1
from $0.5573$ to $0.4880$ and raising forecasting MAE by $7.3\%$, $56.1\%$ and
$119.1\%$ on imaging, clinical measures and genetics. Removing the gate is the
next most damaging, and removing coupling or replacing the symmetric product by
the asymmetric gate costs less. Two qualifications apply. The margins between
the full model and the three non-IMEX variants are smaller than the seed
standard deviations, which reach $\pm0.034$ on macro F1 and $\pm0.112$ on
AUPRC, so only the discretization effect is resolved at this sample size. The
asymmetric gate breaks the symmetry required by Lemma~\ref{lem:psd}, but we
observed no divergence, so that guarantee acts as a safeguard rather than a
practical necessity at this scale.

\begin{table}[t]
\caption{Ablation study evaluating the contribution of coupling, gating, and
discretisation choices on OASIS-3.}
\label{tab:ablation_oasis}
\centering
\small
\setlength{\tabcolsep}{4pt}
\resizebox{\textwidth}{!}{%
\begin{tabular}{@{}p{0.30\textwidth}cccccc@{}}
\toprule
& \multicolumn{1}{c}{Staging} & \multicolumn{2}{c}{Landmark Prediction}
& \multicolumn{3}{c}{Longitudinal Forecasting} \\
\cmidrule(lr){2-2}\cmidrule(lr){3-4}\cmidrule(lr){5-7}
Variants & Macro F1 \up & AUROC \up & AUPRC \up
& Img. MAE \dn & Clin. MAE \dn & Gen. MAE \dn \\
\midrule
\textit{w/o} coupling ($C \equiv 0$)
& $0.5422_{\pm 0.0285}$ & $0.9678_{\pm 0.0203}$ & $0.7677_{\pm 0.0802}$ & $0.7647_{\pm 0.0197}$ & $0.3947_{\pm 0.0097}$ & $0.3214_{\pm 0.0056}$ \\
\textit{w} coupling ungated ($g_{kj}^{n} \equiv 1$)
& $0.5341_{\pm 0.0339}$ & $0.9602_{\pm 0.0350}$ & $0.7324_{\pm 0.0991}$ & $0.7798_{\pm 0.0201}$ & $0.4075_{\pm 0.0116}$ & $0.3392_{\pm 0.0076}$ \\
\textit{w} asymmetric gate ($g_{kj}^{n}=a_j^n$)
& $0.5493_{\pm 0.0341}$ & $0.9661_{\pm 0.0361}$ & $0.7586_{\pm 0.1121}$ & $0.7702_{\pm 0.0201}$ & $0.3989_{\pm 0.0116}$ & $0.3298_{\pm 0.0076}$ \\
\textit{w} IMEX discretization
& $0.4880_{\pm 0.0425}$ & $0.9535_{\pm 0.0308}$ & $0.7069_{\pm 0.1782}$ & $0.8173_{\pm 0.0255}$ & $0.6110_{\pm 0.0265}$ & $0.7032_{\pm 0.0131}$ \\
\midrule
\textbf{CAMOS}
& $\best{0.5573}_{\pm \best{0.0260}}$
& $\best{0.9707}_{\pm \best{0.0210}}$
& $\best{0.7777}_{\pm \best{0.0626}}$
& $\best{0.7614}_{\pm \best{0.0190}}$
& $\best{0.3913}_{\pm \best{0.0111}}$
& $\best{0.3210}_{\pm \best{0.0045}}$ \\
\bottomrule
\end{tabular}%
}
\end{table}

\end{document}

%% file: math_commands.tex
\usepackage{amsmath,amsfonts,bm}

\def\eqref#1{equation~\ref{#1}}

\def\1{\bm{1}}

\DeclareMathAlphabet{\mathsfit}{\encodingdefault}{\sfdefault}{m}{sl}
\SetMathAlphabet{\mathsfit}{bold}{\encodingdefault}{\sfdefault}{bx}{n}



%% file: figures/fig_separation.tex
\begin{tikzpicture}[
  font=\scriptsize,
  >={Stealth[length=1.5mm]},
  st/.style={draw, circle, minimum size=6.2mm, inner sep=0pt, font=\tiny, fill=white},
  op/.style={draw, rounded corners=1pt, minimum width=7mm, minimum height=4.6mm,
             inner sep=1pt, font=\tiny},
  gate/.style={draw, rectangle, minimum size=3.2mm, inner sep=0pt, font=\tiny,
               fill=orange!25, draw=orange!70!black},
  ar/.style={->, line width=0.45pt}
]

\begin{scope}
  \node[font=\scriptsize\bfseries, anchor=west] at (-0.35,1.95)
    {(a) Input gating: class $\mathcal{A}$ (LinOSS, S4/S5, LRU)};

  \foreach \i in {0,1,2,3}{
    \node[op, fill=gray!12] (T\i) at (\i*1.45,0.55) {$\mathcal{T}_{\the\numexpr\i+1\relax}$};
  }
  \foreach \i [evaluate=\i as \j using int(\i+1)] in {0,1,2}{
    \draw[ar] (T\i) -- (T\j);
  }
  \draw[ar] (-0.95,0.55) -- (T0);
  \draw[ar] (T3) -- ++(0.85,0) node[right, font=\tiny] {$o$};

  \foreach \i in {0,1,2,3}{
    \node[gate] (g\i) at (\i*1.45,-0.55) {$a^{\the\numexpr\i+1\relax}$};
    \draw[ar] (g\i) -- (T\i);
    \draw[ar] (\i*1.45,-1.35) -- node[right, font=\tiny, pos=0.15]{$u^{\the\numexpr\i+1\relax}$} (g\i);
  }
  \node[font=\tiny, text=black!70, align=center, anchor=north] at (2.2,-1.72)
    {$a$ multiplies the \emph{forcing} only;\\
     $\mathcal{T}_n$ is independent of $a$};
  \node[draw=black!35, rounded corners=2pt, fill=gray!6, inner sep=3pt, anchor=north,
        align=center, font=\tiny] at (2.2,-2.45)
    {$\displaystyle o=\sum_{k,m} a_k^m\,\psi_{k,m}(u)$\\[1.5pt]
     \textbf{additive} in the availability pattern};
\end{scope}

\begin{scope}[shift={(7.6,0)}]
  \node[font=\scriptsize\bfseries, anchor=west] at (-0.35,1.95)
    {(b) Transition gating: CAMOS};

  \foreach \i in {0,1,2,3}{
    \node[op, fill=cyan!14, draw=cyan!60!black, line width=0.6pt]
      (U\i) at (\i*1.45,0.55) {$\mathcal{T}_{\the\numexpr\i+1\relax}(a^{\the\numexpr\i+1\relax})$};
  }
  \foreach \i [evaluate=\i as \j using int(\i+1)] in {0,1,2}{
    \draw[ar] (U\i) -- (U\j);
  }
  \draw[ar] (-0.95,0.55) -- (U0);
  \draw[ar] (U3) -- ++(0.85,0) node[right, font=\tiny] {$o$};

  \foreach \i in {0,1,2,3}{
    \node[gate] (h\i) at (\i*1.45,-0.55) {$a^{\the\numexpr\i+1\relax}$};
    \draw[ar] (h\i) -- (U\i);
    \draw[ar, orange!75!black, line width=0.7pt] (h\i) to[out=35,in=-35] (U\i.south east);
    \draw[ar] (\i*1.45,-1.35) -- node[right, font=\tiny, pos=0.15]{$u^{\the\numexpr\i+1\relax}$} (h\i);
  }
  \node[font=\tiny, text=orange!60!black, align=center, anchor=north] at (2.2,-1.72)
    {$a$ gates the coupling \emph{inside} $P_n$,\\
     so $\mathcal{T}_n$ itself depends on $a^n$};
  \node[draw=cyan!60!black, rounded corners=2pt, fill=cyan!8, inner sep=3pt, anchor=north,
        align=center, font=\tiny] at (2.2,-2.45)
    {$\displaystyle o=\!\!\prod_{n}\mathcal{T}_n(a^n)x^0+\dots$\\[1.5pt]
     \textbf{non-additive}: mixed differences $\neq 0$};
\end{scope}

\end{tikzpicture}

%% file: iclr2027_conference.bib
@inproceedings{rusch2025linoss,
  title     = {Oscillatory State-Space Models},
  author    = {Rusch, T. Konstantin and Rus, Daniela},
  booktitle = {The Thirteenth International Conference on Learning Representations (ICLR)},
  year      = {2025},
  url       = {https://openreview.net/forum?id=GRMfXcAAFh}
}

@article{boyer2025dlinoss,
  title   = {Learning to Dissipate Energy in Oscillatory State-Space Models},
  author  = {Boyer, Jared and Rusch, T. Konstantin and Rus, Daniela},
  journal = {arXiv preprint arXiv:2505.12171},
  year    = {2025},
  url     = {https://arxiv.org/abs/2505.12171}
}

@inproceedings{gu2022s4,
  title     = {Efficiently Modeling Long Sequences with Structured State Spaces},
  author    = {Gu, Albert and Goel, Karan and R{\'e}, Christopher},
  booktitle = {International Conference on Learning Representations (ICLR)},
  year      = {2022},
  url       = {https://openreview.net/forum?id=uYLFoz1vlAC}
}

@inproceedings{smith2023s5,
  title     = {Simplified State Space Layers for Sequence Modeling},
  author    = {Smith, Jimmy T.H. and Warrington, Andrew and Linderman, Scott},
  booktitle = {The Eleventh International Conference on Learning Representations (ICLR)},
  year      = {2023},
  url       = {https://openreview.net/forum?id=Ai8Hw3AXqks}
}

@inproceedings{orvieto2023lru,
  title     = {Resurrecting Recurrent Neural Networks for Long Sequences},
  author    = {Orvieto, Antonio and Smith, Samuel L and Gu, Albert and Fernando, Anushan and Gulcehre, Caglar and Pascanu, Razvan and De, Soham},
  booktitle = {International Conference on Machine Learning (ICML)},
  pages     = {26670--26698},
  year      = {2023}
}

@inproceedings{gu2024mamba,
  title     = {Mamba: Linear-Time Sequence Modeling with Selective State Spaces},
  author    = {Gu, Albert and Dao, Tri},
  booktitle = {First Conference on Language Modeling (COLM)},
  year      = {2024},
  url       = {https://arxiv.org/abs/2312.00752}
}

@inproceedings{dao2024mamba2,
  title     = {Transformers are {SSM}s: Generalized Models and Efficient Algorithms Through Structured State Space Duality},
  author    = {Dao, Tri and Gu, Albert},
  booktitle = {International Conference on Machine Learning (ICML)},
  year      = {2024}
}

@inproceedings{gu2020hippo,
  title     = {{HiPPO}: Recurrent Memory with Optimal Polynomial Projections},
  author    = {Gu, Albert and Dao, Tri and Ermon, Stefano and Rudra, Atri and R{\'e}, Christopher},
  booktitle = {Advances in Neural Information Processing Systems (NeurIPS)},
  volume    = {33},
  pages     = {1474--1487},
  year      = {2020}
}

@inproceedings{chen2018neuralode,
  title     = {Neural Ordinary Differential Equations},
  author    = {Chen, Ricky T. Q. and Rubanova, Yulia and Bettencourt, Jesse and Duvenaud, David},
  booktitle = {Advances in Neural Information Processing Systems (NeurIPS)},
  volume    = {31},
  year      = {2018}
}

@inproceedings{rubanova2019latentode,
  title     = {Latent Ordinary Differential Equations for Irregularly-Sampled Time Series},
  author    = {Rubanova, Yulia and Chen, Ricky T. Q. and Duvenaud, David K},
  booktitle = {Advances in Neural Information Processing Systems (NeurIPS)},
  volume    = {32},
  year      = {2019}
}

@inproceedings{kidger2020ncde,
  title     = {Neural Controlled Differential Equations for Irregular Time Series},
  author    = {Kidger, Patrick and Morrill, James and Foster, James and Lyons, Terry},
  booktitle = {Advances in Neural Information Processing Systems (NeurIPS)},
  volume    = {33},
  pages     = {6696--6707},
  year      = {2020}
}

@inproceedings{walker2024logncde,
  title     = {Log Neural Controlled Differential Equations: The {L}ie Brackets Make a Difference},
  author    = {Walker, Benjamin and McLeod, Andrew D. and Qin, Tiexin and Cheng, Yichuan and Li, Haoliang and Lyons, Terry},
  booktitle = {International Conference on Machine Learning (ICML)},
  year      = {2024}
}

@inproceedings{debrouwer2019gruodebayes,
  title     = {{GRU-ODE-Bayes}: Continuous Modeling of Sporadically-Observed Time Series},
  author    = {De Brouwer, Edward and Simm, Jaak and Arany, Adam and Moreau, Yves},
  booktitle = {Advances in Neural Information Processing Systems (NeurIPS)},
  volume    = {32},
  year      = {2019}
}

@inproceedings{shukla2021mtan,
  title     = {Multi-Time Attention Networks for Irregularly Sampled Time Series},
  author    = {Shukla, Satya Narayan and Marlin, Benjamin M.},
  booktitle = {International Conference on Learning Representations (ICLR)},
  year      = {2021},
  url       = {https://openreview.net/forum?id=4c0J6lwQ4_}
}

@inproceedings{zhang2022raindrop,
  title     = {Graph-Guided Network for Irregularly Sampled Multivariate Time Series},
  author    = {Zhang, Xiang and Zeman, Marko and Tsiligkaridis, Theodoros and Zitnik, Marinka},
  booktitle = {International Conference on Learning Representations (ICLR)},
  year      = {2022},
  url       = {https://openreview.net/forum?id=Kwm8I7dU-l5}
}

@article{che2018grud,
  title   = {Recurrent Neural Networks for Multivariate Time Series with Missing Values},
  author  = {Che, Zhengping and Purushotham, Sanjay and Cho, Kyunghyun and Sontag, David and Liu, Yan},
  journal = {Scientific Reports},
  volume  = {8},
  number  = {1},
  pages   = {6085},
  year    = {2018},
  doi     = {10.1038/s41598-018-24271-9}
}

@inproceedings{yao2024drfuse,
  title     = {{DrFuse}: Learning Disentangled Representation for Clinical Multi-Modal Fusion with Missing Modality and Modal Inconsistency},
  author    = {Yao, Wenfang and Yin, Kejing and Cheung, William K. and Liu, Jia and Qin, Jing},
  booktitle = {Proceedings of the AAAI Conference on Artificial Intelligence},
  volume    = {38},
  pages     = {16416--16424},
  year      = {2024},
  doi       = {10.1609/aaai.v38i15.29578}
}

@inproceedings{hayat2022medfuse,
  title     = {{MedFuse}: Multi-modal Fusion with Clinical Time-series Data and Chest {X}-ray Images},
  author    = {Hayat, Nasir and Geras, Krzysztof J. and Shamout, Farah E.},
  booktitle = {Proceedings of the 7th Machine Learning for Healthcare Conference},
  series    = {Proceedings of Machine Learning Research},
  volume    = {182},
  pages     = {479--503},
  year      = {2022},
  publisher = {PMLR}
}

@inproceedings{li2024coupledmamba,
  title     = {Coupled Mamba: Enhanced Multimodal Fusion with Coupled State Space Model},
  author    = {Li, Wenbing and Zhou, Hang and Yu, Junqing and Song, Zikai and Yang, Wei},
  booktitle = {Advances in Neural Information Processing Systems (NeurIPS)},
  volume    = {37},
  year      = {2024},
  url       = {https://openreview.net/forum?id=UXEo3uNNIX}
}

@inproceedings{ma2021smil,
  title     = {{SMIL}: Multimodal Learning with Severely Missing Modality},
  author    = {Ma, Mengmeng and Ren, Jian and Zhao, Long and Tulyakov, Sergey and Wu, Cathy and Peng, Xi},
  booktitle = {Proceedings of the AAAI Conference on Artificial Intelligence},
  volume    = {35},
  pages     = {2302--2310},
  year      = {2021}
}

@inproceedings{zhang2022m3care,
  title     = {{M3Care}: Learning with Missing Modalities in Multimodal Healthcare Data},
  author    = {Zhang, Chaohe and Chu, Xu and Ma, Liantao and Zhu, Yinghao and Wang, Yasha and Zhao, Jiangtao and Wang, Junfeng},
  booktitle = {Proceedings of the 28th ACM SIGKDD Conference on Knowledge Discovery and Data Mining},
  pages     = {2418--2428},
  year      = {2022}
}

@inproceedings{wang2023shaspec,
  title     = {Multi-modal Learning with Missing Modality via Shared-Specific Feature Modelling},
  author    = {Wang, Hu and Chen, Yuanhong and Ma, Congbo and Avery, Jodie and Hull, Louise and Carneiro, Gustavo},
  booktitle = {Proceedings of the IEEE/CVF Conference on Computer Vision and Pattern Recognition (CVPR)},
  pages     = {15878--15887},
  year      = {2023}
}

@inproceedings{joze2020mmtm,
  title     = {{MMTM}: Multimodal Transfer Module for {CNN} Fusion},
  author    = {Joze, Hamid Reza Vaezi and Shaban, Amirreza and Iuzzolino, Michael L. and Koishida, Kazuhito},
  booktitle = {Proceedings of the IEEE/CVF Conference on Computer Vision and Pattern Recognition (CVPR)},
  pages     = {13289--13299},
  year      = {2020}
}

@inproceedings{polsterl2021daft,
  title     = {Combining 3{D} Image and Tabular Data via the Dynamic Affine Feature Map Transform},
  author    = {P{\"o}lsterl, Sebastian and Wolf, Tom Nuno and Wachinger, Christian},
  booktitle = {Medical Image Computing and Computer Assisted Intervention (MICCAI)},
  pages     = {688--698},
  year      = {2021},
  publisher = {Springer}
}

@article{petersen2010adni,
  title   = {{Alzheimer's Disease Neuroimaging Initiative (ADNI)}: Clinical Characterization},
  author  = {Petersen, Ronald C. and Aisen, Paul S. and Beckett, Laurel A. and Donohue, Michael C. and Gamst, Anthony C. and Harvey, Danielle J. and Jack, Clifford R. and Jagust, William J. and Shaw, Leslie M. and Toga, Arthur W. and Trojanowski, John Q. and Weiner, Michael W.},
  journal = {Neurology},
  volume  = {74},
  number  = {3},
  pages   = {201--209},
  year    = {2010},
  doi     = {10.1212/WNL.0b013e3181cb3e25}
}

@article{jack2010cascade,
  title   = {Hypothetical Model of Dynamic Biomarkers of the {A}lzheimer's Pathological Cascade},
  author  = {Jack, Clifford R. and Knopman, David S. and Jagust, William J. and Shaw, Leslie M. and Aisen, Paul S. and Weiner, Michael W. and Petersen, Ronald C. and Trojanowski, John Q.},
  journal = {The Lancet Neurology},
  volume  = {9},
  number  = {1},
  pages   = {119--128},
  year    = {2010},
  doi     = {10.1016/S1474-4422(09)70299-6}
}

@article{jack2013updated,
  title   = {Tracking Pathophysiological Processes in {A}lzheimer's Disease: An Updated Hypothetical Model of Dynamic Biomarkers},
  author  = {Jack, Clifford R. and Knopman, David S. and Jagust, William J. and Petersen, Ronald C. and Weiner, Michael W. and Aisen, Paul S. and Shaw, Leslie M. and Vemuri, Prashanthi and Wiste, Heather J. and Weigand, Stephen D. and Lesnick, Timothy G. and Pankratz, Vernon S. and Donohue, Michael C. and Trojanowski, John Q.},
  journal = {The Lancet Neurology},
  volume  = {12},
  number  = {2},
  pages   = {207--216},
  year    = {2013},
  doi     = {10.1016/S1474-4422(12)70291-0}
}

@article{hardy2002amyloid,
  title   = {The Amyloid Hypothesis of {A}lzheimer's Disease: Progress and Problems on the Road to Therapeutics},
  author  = {Hardy, John and Selkoe, Dennis J.},
  journal = {Science},
  volume  = {297},
  number  = {5580},
  pages   = {353--356},
  year    = {2002},
  doi     = {10.1126/science.1072994}
}

@inproceedings{vaswani2017attention,
  title     = {Attention Is All You Need},
  author    = {Vaswani, Ashish and Shazeer, Noam and Parmar, Niki and Uszkoreit, Jakob and Jones, Llion and Gomez, Aidan N. and Kaiser, {\L}ukasz and Polosukhin, Illia},
  booktitle = {Advances in Neural Information Processing Systems (NeurIPS)},
  volume    = {30},
  year      = {2017}
}

@inproceedings{dauphin2017glu,
  title     = {Language Modeling with Gated Convolutional Networks},
  author    = {Dauphin, Yann N. and Fan, Angela and Auli, Michael and Grangier, David},
  booktitle = {International Conference on Machine Learning (ICML)},
  pages     = {933--941},
  year      = {2017}
}

@article{hendrycks2016gelu,
  title   = {{G}aussian Error Linear Units ({GELU}s)},
  author  = {Hendrycks, Dan and Gimpel, Kevin},
  journal = {arXiv preprint arXiv:1606.08415},
  year    = {2016}
}

@article{blelloch1990prefix,
  title   = {Prefix Sums and Their Applications},
  author  = {Blelloch, Guy E.},
  journal = {Technical Report CMU-CS-90-190, School of Computer Science, Carnegie Mellon University},
  year    = {1990}
}

@article{kogge1973parallel,
  title   = {A Parallel Algorithm for the Efficient Solution of a General Class of Recurrence Equations},
  author  = {Kogge, Peter M. and Stone, Harold S.},
  journal = {IEEE Transactions on Computers},
  volume  = {C-22},
  number  = {8},
  pages   = {786--793},
  year    = {1973}
}

@article{martin2018parallel,
  title   = {Parallelizing Linear Recurrent Neural Nets Over Sequence Length},
  author  = {Martin, Eric and Cundy, Chris},
  journal = {arXiv preprint arXiv:1709.04057},
  year    = {2017}
}

@book{horn2012matrix,
  title     = {Matrix Analysis},
  author    = {Horn, Roger A. and Johnson, Charles R.},
  edition   = {2nd},
  publisher = {Cambridge University Press},
  year      = {2012}
}

@article{rubin1976missing,
  title   = {Inference and Missing Data},
  author  = {Rubin, Donald B.},
  journal = {Biometrika},
  volume  = {63},
  number  = {3},
  pages   = {581--592},
  year    = {1976}
}

@article{lamontagne2019oasis3,
  title   = {{OASIS-3}: Longitudinal Neuroimaging, Clinical, and Cognitive Dataset for Normal Aging and {A}lzheimer Disease},
  author  = {LaMontagne, Pamela J. and Benzinger, Tammie L.S. and Morris, John C. and Keefe, Sarah and Hornbeck, Russ and Xiong, Chengjie and Grant, Elizabeth and Hassenstab, Jason and Moulder, Krista and Vlassenko, Andrei G. and Raichle, Marcus E. and Cruchaga, Carlos and Marcus, Daniel},
  journal = {medRxiv},
  year    = {2019},
  doi     = {10.1101/2019.12.13.19014902}
}

@article{paszke2019pytorch,
  title   = {{PyTorch}: An Imperative Style, High-Performance Deep Learning Library},
  author  = {Paszke, Adam and Gross, Sam and Massa, Francisco and Lerer, Adam and Bradbury, James and Chanan, Gregory and Killeen, Trevor and Lin, Zeming and Gimelshein, Natalia and Antiga, Luca and Desmaison, Alban and K{\"o}pf, Andreas and Yang, Edward and DeVito, Zachary and Raison, Martin and Tejani, Alykhan and Chilamkurthy, Sasank and Steiner, Benoit and Fang, Lu and Bai, Junjie and Chintala, Soumith},
  journal = {Advances in Neural Information Processing Systems (NeurIPS)},
  volume  = {32},
  year    = {2019}
}
